%% file: main.tex
\documentclass[12pt]{article}

\usepackage[margin=1in]{geometry}
\usepackage{graphicx}
\usepackage{amsmath}
\usepackage{amsfonts}

\usepackage{subcaption}
\usepackage{multirow}

\usepackage{authoraftertitle}
\usepackage{lastpage}
\usepackage{fancyhdr}
\usepackage{amsthm}
\theoremstyle{plain}

\newtheorem{thm}{Theorem}
\newtheorem{lem}{Lemma}
\newtheorem{prop}{Proposition}

\theoremstyle{definition}
\newtheorem{defn}{Definition}

\theoremstyle{remark}

\usepackage{hyperref}
\hypersetup{
	colorlinks = false,
	pdfborder = {0 0 0.5 [3 2]},
	linkbordercolor = [rgb]{0.8500, 0.3250, 0.0980}, 
	citebordercolor = [rgb]{0.49, 0.72, 0.19},
}

\usepackage{bm}
\usepackage{bbm}
\usepackage{algorithm}
\usepackage{datetime} %

\usepackage{natbib}
\usepackage{algorithm} %
\usepackage{enumitem} %
\usepackage{setspace}
\usepackage{tabularx}
\newcolumntype{Y}{>{\centering\arraybackslash}X}

\DeclareMathOperator*{\argmin}{argmin}

\def\keywords#1{\par
	\vspace*{8pt}
	{{\leftskip18pt\rightskip\leftskip
	\noindent{\it Keywords}\/:\ #1\par}}\par}

\usepackage{lipsum}
 
\newcommand\blfootnote[1]{%
  \begingroup
  \renewcommand\thefootnote{}\footnote{#1}%
  \addtocounter{footnote}{-1}%
  \endgroup
}
\title{The Success of AdaBoost and Its Application in Portfolio Management}
\author{Yijian Chuan$^{\ddag}$, Chaoyi Zhao$^{\ddag}$, Zhenrui He$^{\ddag}$ and Lan Wu$^{\dag,}$$^\ast$\\
\small
$^{\ddag}$School of Mathematical Sciences, Peking University, Beijing, China\\
\small
$^{\dag}$LMEQF, School of Mathematical Sciences, Peking University, Beijing, China
\blfootnote{Corresponding author. lwu@pku.edu.cn}
}
\date{}

\begin{document}
\captionsetup[figure]{name={Fig.},labelsep=period}

\maketitle
\doublespacing
\thispagestyle{fancy}

\begin{abstract}
We develop a novel approach to explain why AdaBoost is a successful classifier. By introducing a measure of the influence of the noise points (ION) in the training data for the binary classification problem, we prove that there is a strong connection between the ION and the test error. We further identify that the ION of AdaBoost decreases as the iteration number or the complexity of the base learners increases. We confirm that it is impossible to obtain a consistent classifier without deep trees as the base learners of AdaBoost in some complicated situations. We apply AdaBoost in portfolio management via empirical studies in the Chinese market, which corroborates our theoretical propositions. 

\end{abstract}

\keywords{AdaBoost, interpolation, noise point, base learner, equal-weighted portfolio}

\newpage

\input{secs/intro}

\input{secs/strategy_backgound}

\input{secs/limit_of_adaboost}

\input{secs/empirical_results}

\input{secs/conclusion}

\appendix

\bibliography{mybib}
\end{document}

%% file: secs/intro.tex
\section{Introduction}

Equal-weighted portfolios are one of the most important strategies in portfolio management. They are portfolios with weights equally distributed across the selected securities in the long and/or short positions. In academic research, numerous studies have suggested that equal-weighted portfolios have a better out-of-sample performance than other portfolios (e.g., \citealt{jobson1981}; \citealt{james2003bias_variance}; \citealt{demiguel2007_1n}).
\citet{michaud1989markowitz} and \citet{demiguel2007_1n} argued that, the equal-weighted strategies do not suffer from the estimation error of the covariance matrix, which is vulnerable to outliers \citep{tu2011}.
In industry, equal-weighted portfolios are popular across portfolio management in practice, particularly in the hedge funds.
The MSCI has issued many equal-weighted indexes, which are ``some of the oldest and best-known factor strategies that have aimed to identify specific characteristics of stocks generating excess return''\footnote{\url{https://www.msci.com/msci-equal-weighted-indexes}}.

The core of constructing equal-weighted portfolios is to forecast the long and/or short positions, which is a classification problem. 
Machine learning usually outstands when dealing with classification problems, and the research of machine learning on classification is diverse. 
\citet{DGL1996} lists numerous traditional pattern recognition research in computer science, where the pattern recognition is an alias for the classification problem. 
\citet{book:esl} analyzes and summaries many machine learning methods, among which include lots of classification methods, such as the Linear/Quadratic Discriminant Analysis, Support Vector Machine, Boosting, Random Forest, etc. 
The application of machine learning in classification succeeds in various fields, such as email spam identification, handwritten digit recognition, etc.
In portfolio management, \citet{de2018ml} explained how to apply machine learning to managing funds for some of the most demanding institutional investors.
Specifically,
\citet{creamer2005AdaBoost, creamer2010automated} applied the Boosting method in finance, and revealed its practical value.
\citet{creamer2012model} used LogitBoost in high-frequency data for Euro futures, and generated positive returns.  
\citet{wang2012DB} invented the N-LASR (non-linear adaptive style rotation) model by applying AdaBoost in stock's factor strategy.
They wisely incorporated benefits of different factors by the N-LASR model, and the empirical study on the component stocks in Russel 3000 showed a significant risk-adjusted portfolio return.
\citet{fievet2018trees} proposed a decision tree forecasting model and applied it to S\&P 500, which is capable of capturing arbitrary interaction patterns and generating positive returns.
\citet{rasekhschaffe2019_ML_stock_selection} provided an example of the machine learning techniques to forecast the cross-section of stock returns.
\citet{gu2018_ML_asset_pricing} and \citet{dhondt2020} gave comprehensive analysis of machine learning methods for the canonical problem of empirical asset pricing, attributing their predicted gains to the non-linearity.

AdaBoost is a classification method in machine learning, inspiring a tremendous amount of innovations. 
AdaBoost has developed for more than two decades since 
\citet{freund1996adaboost}. %
Since it is less prone to overfit, Breiman praised AdaBoost---``the best off-the-shelf classifier in the world''---at the 1996 NeurIPS conference \citep{friedman2000additive}.
AdaBoost has made a significant impact on machine learning and statistics. 
To explain AdaBoost's overfitting resistance, \citet{schapire1998boosting} proposed the ``margin'' theory.
Meanwhile, \citet{breiman1998arcing, breiman1999prediction} and 
\citet{ friedman2000additive} discovered the fact that AdaBoost is equivalent to an optimization algorithm, so \citet{friedman2001greedy} put forward the Gradient Boosting. 
Inspired by AdaBoost, \citet{breiman2001random} invented the Random Forest(RF) and believed that there are some similarities between RF and AdaBoost. 
Subsequently, people generalized the Boosting methods, in which two Boosting algorithms are widely-adopted---the XGBoost \citep{XGBoost} and the LightGBM \citep{LGBM}. 
Till now, the family of Boosting flourishes, and becomes a considerable part in machine learning. 

Although there are many studies explaining why AdaBoost is a successful method, people are still curious about its excellent achievement till now.
\citet{wyner2003OnBA, mease2008contrary} believed that the available interpretation of AdaBoost is ``incomplete'', particularly on the explanation of its overfitting resistance property.  
\citet[p. 10]{wyner2017AdaBoost_RandomForest} introduced a novel perspective on AdaBoost and RF, and conjectured that their success could be explained by the ``spiked-smooth'', where \textit{spiked} is ``in the sense that it is allowed to adapt in very small regions to noise points'', and \textit{smooth} is ``in the sense that it has been produced through averaging''. 
In other words, AdaBoost is a self-averaging interpolating method, localizing the effect of the noise points as the iteration number increases. 

Our work is motivated by the questions from the industry: 
``May machine learning strategies outperform other traditional strategies in quantitative investment? Why and how do they work?''
\citet{de2018ml} gave a comprehensive and systematic approach to apply machine learning methods,
and highly appraised Boosting:
``We explored a number of standard black-box approaches. Among machine learning methods, we found gradient tree boosting to be more effective than others.'' 
Besides, \cite{wang2012DB} applied AdaBoost to select and combine factors with consistent and interpretative performance, 
and \citet{zhang2019boosting} proposed a Boosting method to compose portfolios which performs well.
These findings answered the first question.
There is limited research concerning the mechanism or the interpretability of machine learning in portfolio management.
However, interpretability is essential in investment \citep{feng2017zoo}. 
In detail, 
\citet[p. 37]{harvey2016and} argued:
``... a factor developed from first [economic] principles should have a lower threshold t-statistic than a factor that is discovered as a purely empirical exercise.''
\citet{harvey2017presidential} proposed an example. They constructed portfolios based on the first, second, and third letters of the ticker symbols, gaining significant excess returns.
Nevertheless, most people would not like to adopt this symbol-based portfolio, as they implied. Thus, without interpretability, portfolio investment is vulnerable. We must pay more attention to the second/third questions.

To answer the ``why and how'' questions,
we should investigate AdaBoost in the framework of statistics to find a theoretical explanation of its outperformance and apply them in portfolio management. 
\citet{wyner2017AdaBoost_RandomForest} pointed out that:
``The statistics community was especially confounded by two properties of AdaBoost: 1) interpolation (perfect prediction in sample) was achieved after relatively few iterations, 2) generalization error continues to drop even after interpolation is achieved and maintained.''
They innovated the concept of ``spiked-smooth'' classifier created by a self-averaging and interpolating algorithm.
They conjectured that the ``spiked-smooth'' property renders the success of AdaBoost, and provided many delicate examples by simulation to support their viewpoints.
Thus, we would like to narrow the gap between the theory and the simulation by strengthening their work from a statistical perspective. 
To explain the ``spiked-smooth'' mathematically, we need to distinguish the signal and the noise within the training set in a statistical framework first.
Then, one should connect the relationship between the ``spiked-smooth'' and the test error, 
explaining the property of overfitting resistance.

In addition, \citet{wyner2017AdaBoost_RandomForest} pointed out that ``boosting should be used like random forests: with large decision trees, without regularization or early stopping''. 
The point is that larger and deeper decision trees are preferred to be used as the ``weak'' classifiers (base learners) of AdaBoost, since they can both ``interpolate'' the training set and realize the goal of ``spiked-smooth''. 
This point contradicts with the common sense about machine learning and statistics, and statisticians usually believe complexity leads to overfitting. 
Therefore, we wonder if the AdaBoost method can boost shallow trees when the true model is very complicated, just as we cannot ``make bricks without straw''. We try to find out that under what populations will the AdaBoost method be unable to achieve good performance if the base learners are very weak. We want to demonstrate the viewpoints of \citet{wyner2017AdaBoost_RandomForest} in a  mathematical framework.

In this paper, we show that how AdaBoost can dig out more non-linear information in the training set without increasing the test error.
Our work is composed of three parts.
First, to concrete the abstract concept ``spiked-smooth'' into a measurable value, 
we define a measure of the influence of the noise points in the training set for a given method. The measure can also be regarded as a measure of the localization of the given method.
We discover the connection between the measure and the out-of-sample performance. That is, under certain conditions, if the influence of the noise points is not essential, then the test error will be low.
A toy example clarifies the theorem, intuitively illustrating the influence of the noise and explaining why it controls the test error.
For AdaBoost, 
we show that, 
as the number of iterations increases or the depth of the base learners grows, it becomes more robust to the influence of the noise, and thus lead to a lower test error.
Therefore, 
we give a theoretical explanation about why AdaBoost has a good performance without overfitting in noisy training sets.

Second, 
we confirm that it is a better choice to use deeper/larger decision trees as base learners of AdaBoost in the sense of digging out complex information. 
Specifically,
we propose several counterexamples that AdaBoost based on shallow decision trees fails to handle, even after iterating infinite times.
We generalize the results and indicate that AdaBoost based on shallow decision trees would fail in recognizing a certain kind of information, while the one based on deep decision trees could easily solve out.
Therefore, these findings suggest that AdaBoost based on deep decision trees maybe better.

Third, the empirical studies in the Chinese market corroborate our theoretical propositions. The theoretical results about the interpolation and the localization of AdaBoost in the previous parts of this paper is verified by constructing an optimal portfolio strategy. Besides, the result also illustrates the good performance of the equal-weighted portfolio generated by the selected optimal classifier trained by AdaBoost.

The outline of this paper is as follows.
Section~\ref{sec:ION and AdaBoost} introduces a measure of ``spiked-smooth'', illustrates the relationship between the measure and the test error, and explains the success of AdaBoost.
Section~\ref{sec:XOR_counter} identifies that AdaBoost based on deep trees can dig out more information, while the one based on shallow trees fails.
Section~\ref{sec:empirical} provides empirical studies of AdaBoost in the Chinese stock market.
Section~\ref{sec:conclusion} concludes.

%% file: secs/strategy_backgound.tex
\section{The influence of the noise points and AdaBoost}\label{sec:ION and AdaBoost}

In this section, we give a strict definition for the ``spiked-smooth'' suggested by \citet{wyner2017AdaBoost_RandomForest} in the framework of the Bayes classifier. 
Under the framework, we explain the success of AdaBoost by developing a concrete measure.

First, we describe a background model of the binary classification problem and the Bayes classifier, and define the signal/noise points for a given training set. 
Based on these concepts, we build a bridge between the Bayes classifier and the interpolating classifier.
We define a measure of the influence of the noise points,
and specify its property.
Second,
we explore the connection between the measure and the test error.
Last, we explain the success of AdaBoost as the minimization of the influence of the noise points in the sense of the ``spiked-smooth'', and reveal its potential applications in portfolio management.

\subsection{The model of the binary classification}
\label{subsection: The model of the binary classification}

A prediction model consists of an input vector $X \in \mathcal{X}$, an output $Y \in \mathcal{Y}=\{\pm1\}$, and a prediction classifier $f: \mathcal{X} \rightarrow \mathcal{Y}$.
For simplicity, let us assume that the distribution of $ X $ is absolutely continuous with respect to the Lebesgue measure.
We restate the definition of the Bayes classifier/error \citep[p. 2]{DGL1996} in Definition \ref{def:Bayes} below.
\begin{defn}[Bayes Classifier/Error]\label{def:Bayes}
Given the population $X$ and $Y\in \mathcal{Y}=\{\pm1\}$, the Bayes classifier is
\begin{align*}
f^{B} := \argmin_{f:\mathcal{X}\rightarrow\mathcal{Y}} \mathbb{P}\{f(X) \neq Y\},
\end{align*}
and the minimum is the Bayes error $q$, i.e.,
\begin{align*}
q := \min_{f:\mathcal{X}\rightarrow\mathcal{Y}} \mathbb{P}\{f(X) \neq Y\}.
\end{align*}
\end{defn}

According to the definition, $ f^{B} $ is the classifier minimizing the test error, which can be represented by the conditional distribution $\mathbb{P}_{Y|X}$ in the population. One can show $f^B(\bm{x}) = \mathrm{sign}\left(\mathbb{P}\{Y=1|X=\bm{x}\}-0.5\right) $, when the population satisfies certain canonical conditions.
There is no classifier having lower test error than $f^{B}$.
We give a general representation of the test error of a classifier by the Bayes classifier in Lemma~\ref{lem:f_f_G_Y} while the noise and $X$ are independent.
\begin{lem}\label{lem:f_f_G_Y}
Given the population $ (X, Y) $, the Bayes classifier $ f^{B} $, and the Bayes error $q$,
if $ \mathbf{1}_{\{Y \neq f^{B}(X)\}} $ and $ X $ are independent,
then, for any classifier $f$,
\begin{align}
\mathbb{P}\{f(X)\neq Y\}
=
q
\mathbb{P}_{X}\{f(X)=f^B(X)\}
+
(1-q)
\mathbb{P}_{X}\{f(X)\neq f^B(X)\}
.
\label{equ:f_f_G_Y}
\end{align}
\end{lem}
\begin{proof}
We have
\begin{align*}
&\ \mathbb{P}\{f(X)\neq Y\}
\\
&=\ 
\mathbb{P} \{ f(X) \neq Y, f(X) = f^{B}(X) \} 
+
\mathbb{P} \{ f(X) \neq Y, f(X) \neq f^{B}(X) \} 
\\
&=\ 
\mathbb{P} \{ f^B(X) \neq Y, f(X) = f^{B}(X) \} 
+
\mathbb{P} \{ f^B(X) = Y, f(X) \neq f^{B}(X) \} 
\\
&=\ 
\mathbb{P} \{ f^B(X) \neq Y\} \mathbb{P}_X\{ f(X) = f^{B}(X) \} 
+
\mathbb{P} \{ f^B(X) = Y\} \mathbb{P}_X\{ f(X) \neq f^{B}(X) \} 
\\
&=\ 
q \mathbb{P}_X\{ f(X) = f^{B}(X) \} 
+
(1-q) \mathbb{P}_X\{ f(X) \neq f^{B}(X) \}
.
\end{align*}
\end{proof}

A natural corollary of \eqref{equ:f_f_G_Y} is $\mathbb{P}\{f(X)\neq Y\}
=
q
+
(1-2q)
\mathbb{P}_{X}\{f(X)\neq f^B(X)\}
$.
In other words, $\mathbb{P}\{f(X)\neq Y\}$ is a linear function of $\mathbb{P}_{X}\{f(X)\neq f^B(X)\}$.

Next, we introduce the concept of the signal/noise points of the training set $T$ in the following Definition \ref{def:Signalpoints}.

\begin{defn}[Signal/Noise Points]\label{def:Signalpoints}
G{}iven a training set $T = (\bm{x}_i, y_i)_{i=1}^n$ generated from the population $(X,Y)$ and the Bayes classifier $f^B$, a point $(\bm{x}_i, y_i)$ is a signal point, if $f^{B}(\bm{x}_i) = y_i$; and it is a noise point, if $f^{B}(\bm{x}_i) \neq y_i$.
\end{defn}
In short, 
the Bayes classifier $f^B$ distinguishes the signal/noise points of a training set.
Heuristically, 
the signal points are the points that equal to the output of the Bayes classifier, while the noise points are not. 

We recall the definition of the interpolation classifier proposed by \citet{wyner2017AdaBoost_RandomForest} for coherence.
\begin{defn}[Interpolating Classifier]
\label{defn:interpolation}
A classifier $f$ is an interpolating classifier on the training set $T = (\bm{x}_i, y_i)_{i=1}^n$, 
if $f(\bm{x}_i)=y_i$ for all $(\bm{x}_i, y_i), i=1,\dots,n$.
\end{defn}

Immediately, we can obtain a property of the interpolating classifier, i.e., its training error is $0$ on the training set $T=(\bm{x}_i, y_i)_{i=1}^n$.

Though the Bayes classifier is the best classifier in the sense of minimizing the test error,
it does not necessarily interpolate the given training set $T = (\bm{x}_i, y_i)_{i=1}^n$.
The Bayes classifier $f^B$ violates interpolation at and only at the noise points,
as implied in Definition~\ref{defn:interpolation}.
Thus, 
in view of the training set,
the difference between an interpolating classifier and the Bayes classifier is only on the noise points.
So, we propose a definition of a purified training set of $T$ by converting the noise points into the ``signal'' points.

\begin{defn}[Purified Training Set]
Given a training set $T = (\bm{x}_i, y_i)_{i=1}^n$ from the population $(X,Y)$ and the Bayes classifier $f^B$,
the purified training set of $T$ is defined as $T_p := (\bm{x}_i, f^B(\bm{x}_i))_{i=1}^n$.
\end{defn}

There is no noise point in $T_p$. In other words, the Bayes classifier must interpolate the purified training set $T_p$. We can also rewrite the definition of the purified training set as
\begin{align*}
    T_p = (\bm{x}_i, \theta_i)_{i=1}^{n}, 
    && 
    \theta_i = \begin{cases}
    y_i, & \text{$i$ is a signal point;}\\
    -y_i, & \text{$i$ is a noise point.}\\
    \end{cases}
\end{align*}
The two training sets $T$ and $T_p$ share the same input $\bm{x}_i$'s but different outputs, and the difference between the outputs of the two sets is only on the noise points. 
The purpose is to separate out the influence of the noise points from the whole information contained in the training set $T$.

Last, based on the previous preparations, 
we propose a measure of the influence of the noise points (ION) for a given training set $T$ and a given method $\mathcal{M}$. 
It helps us to compare the properties of different methods, such as one nearest neighbor (1NN) or AdaBoost, on a given training set.

\begin{defn}[ION]\label{defn:ION}
Given the marginal probability measure of $X$ ($\mathbb{P}_X$), we define the influence of noise (ION), a function of the training set $T$ and the method $\mathcal{M}$:
\begin{align}\label{equ:ION}
\mathrm{ION}(\mathcal{M}, T) 
:= 
\mathbb{P}_{X}\left\{\big.f_T (X) \neq f_{T_p}(X)\right\},
\end{align}
where $f_T$ is the classifier trained on the training set $T$ using method $\mathcal{M}$, and $f_{T_p}$ is on the purified training set $T_p$ using method $\mathcal{M}$.
\end{defn}

We interpret Definition~\ref{defn:ION}. 
$\mathcal{M}$ represents a specific method. 
For instance, for 1NN, one can apply it on the training set $T$ and $T_p$, which generates two classifiers $f_T$ and $f_{T_p}$.
Then, by comparing the two classifiers, we can get the value of ION(1NN, $T$). The ION is defined according to two sets: the training set $T$ and the proxy of the training set generated by $f^B$.

Although Definition~\ref{defn:ION} does not require the interpolation of the classifier $f$, 
ION usually characterizes the performance of the method which generates interpolating classifiers on a given training set.
Meanwhile, $0 \leq \mathrm{ION} \leq 1$. If ION is low, then the classifier is robust to the noise points on the training set of the given method, and vice versa.

Interpolation is not necessarily bad, if it subjects to some ``mechanism'' \citep{wyner2017AdaBoost_RandomForest}.
Although some interpolating classifiers ``can be shown to be inconsistent and have poor generalization error in environments with noise'', ``the claim that all interpolating classifiers overfit is problematic''.
The classifiers generated by 1NN or random forest are both interpolating classifiers, 
but their ION may not be the same.
Furthermore, \citet{wyner2017AdaBoost_RandomForest} suggested: ``an interpolated classifier, if sufficiently local, minimizes the influence of noise points in other parts of the data.''
The next question is, what is the relationship between the ION and the so-called ``spiked-smooth'' classifier.

\subsection{The ION and the test error}
\label{subsection: The ION and the test error}

In this section, we reveal the connection between the ION and the test error from theoretical and numerical perspectives.

First, we prove that, under certain conditions, the lower the ION, the lower the test error.

\begin{prop}\label{prop:ION_lower_to_error_lower}
Given the population $(X, Y)$ such that $\mathbf{1}\{Y\neq f^B(X)\}$ is independent of $X$,
let $f^{(1)}_T$ and $f^{(2)}_T$ denote the classifiers generated from two different methods $\mathcal{M}^{(1)}$ and $\mathcal{M}^{(2)}$ on the training set $T$, and $f^{(1)}_{T_p}$ and $f^{(2)}_{T_p}$ denote the ones generated from $\mathcal{M}^{(1)}$ and $\mathcal{M}^{(2)}$ on the purified training set $T_p$, respectively.
If
\begin{align}
    f^{(1)}_{T_p}(x) = f^{(2)}_{T_p}(x) = f^B(x),
    \quad
    a.s.,
    \label{equ:same_before}
\end{align}
and
\begin{align*}
    \mathrm{ION}(\mathcal{M}^{(1)}, T) < \mathrm{ION}(\mathcal{M}^{(2)}, T),
\end{align*}
then
\begin{align}
    \mathbb{P}\left\{f_T^{(1)}(X) \neq Y \right\}
    <
    \mathbb{P}\left\{f_T^{(2)}(X) \neq Y \right\}
    .
    \label{equ:low_err}
\end{align}
\end{prop}

\begin{proof}
Because of \eqref{equ:same_before}, 
we have
\begin{align*}
&\ 
\mathbb{P}_X\left\{f_T^{(1)}(X) \neq f^B(X) \right\}
=
\mathbb{P}_X\left\{f_T^{(1)}(X) \neq f_{T_p}^{(1)}(X) \right\}
=
\mathrm{ION}(\mathcal{M}^{(1)}, T)
\\
<
&\
\mathrm{ION}(\mathcal{M}^{(2)}, T)
=
\mathbb{P}_X\left\{f_T^{(2)}(X) \neq f_{T_p}^{(2)}(X) \right\}
=
\mathbb{P}_X\left\{f_T^{(2)}(X) \neq f^B(X) \right\}
.
\end{align*}
Thus, by Lemma~\ref{lem:f_f_G_Y}, \eqref{equ:low_err} holds.
\end{proof}

Proposition~\ref{prop:ION_lower_to_error_lower} shows that 
the ION controls the test error.
Specifically, it means that, if the two methods could reach the Bayes classifier in the purified training set, then the method with lower ION outperforms the others in the sense of the test error.
For instance, $\mathcal{M}^{(1)}$ might indicate 1NN, while $\mathcal{M}^{(2)}$ indicate AdaBoost.

However, 
the condition \eqref{equ:same_before} is slightly unnatural.
It is so strong that it could only hold in several particular training sets.
We therefore weaken the original condition \eqref{equ:same_before} and establish a more natural condition in Theorem \ref{thm:ION_lower_to_error_lower_limit} below. 

\begin{thm}\label{thm:ION_lower_to_error_lower_limit}
Given the population $(X, Y)$ such that $\mathbf{1}_{\{Y\neq f^B(X)\}}$ is independent of $X$,
let $f^{(1)}_T$ and $f^{(2)}_T$ denote the classifiers generated from two different methods $\mathcal{M}^{(1)}$ and $\mathcal{M}^{(2)}$ on the training set $T$, and $f^{(1)}_{T_p}$ and $f^{(2)}_{T_p}$ denote the ones generated from $\mathcal{M}^{(1)}$ and $\mathcal{M}^{(2)}$ on the purified training set $T_p$, respectively. The size of the training set $T$ or $T_p$ is $n$.
If
\begin{align}
\lim_{n \rightarrow \infty}
\mathbb{P}
\left\{f_{T_p}^{(j)}(X) \neq Y\right\} = q,
\quad
j=1, 2,
\label{equ:fj_consistency}
\end{align}
and
\begin{align}
    \limsup_{n\rightarrow\infty} 
    \left[\bigg.
    \mathrm{ION}(\mathcal{M}^{(1)}, T) - \mathrm{ION}(\mathcal{M}^{(2)}, T)
    \right]
    < 0,
    \label{equ:fj_ION}
\end{align}
then
\begin{align}
    \limsup_{n\rightarrow\infty} 
    \left[\bigg.
    \mathbb{P}\left\{f_T^{(1)}(X) \neq Y \right\} -
    \mathbb{P}\left\{f_T^{(2)}(X) \neq Y \right\}
    \right]
    < 0.
    \label{equ:fj_te}
\end{align}
\end{thm}

\begin{proof}
To begin with,
by Lemma~\ref{lem:f_f_G_Y},
\eqref{equ:fj_consistency} is equivalent to
\begin{align}
\lim_{n \rightarrow \infty}\mathbb{P}_X
\left\{f_{T_p}^{(j)}(X) \neq f^B(X)\right\} = 0,
\quad
j=1, 2,
\label{equ:fj_Tp_fB_0}
\end{align}
Then,
\begin{align*}
&\ 
\mathbb{P}_X\left\{f_T^{(1)}(X) \neq f^B(X) \right\}
-
\mathbb{P}_X\left\{f_T^{(2)}(X) \neq f^B(X) \right\}
\\
&\ 
-
\left(
\mathbb{P}_X\left\{f_T^{(2)}(X) \neq f^B(X), f_T^{(2)}(X) \neq f_{T_p}^{(2)}(X) \right\}
+
\mathbb{P}_X\left\{f_T^{(2)}(X) \neq f^B(X), f_T^{(2)}(X) = f_{T_p}^{(2)}(X) \right\}
\right)
\\
=&\ 
\mathbb{P}_X\left\{f_T^{(1)}(X) \neq f^B(X), f_T^{(1)}(X) \neq f_{T_p}^{(1)}(X) \right\}
-
\mathbb{P}_X\left\{f_T^{(2)}(X) \neq f^B(X), f_T^{(2)}(X) \neq f_{T_p}^{(2)}(X) \right\}
\\
&\ +
\mathbb{P}_X\left\{f_T^{(1)}(X) \neq f^B(X), f_T^{(1)}(X) = f_{T_p}^{(1)}(X) \right\}
-
\mathbb{P}_X\left\{f_T^{(2)}(X) \neq f^B(X), f_T^{(2)}(X) = f_{T_p}^{(2)}(X) \right\}
\\
=&\ 
\mathbb{P}_X\left\{f_{T_p}^{(1)}(X) = f^B(X), f_T^{(1)}(X) \neq f_{T_p}^{(1)}(X) \right\}
-
\mathbb{P}_X\left\{f_{T_p}^{(2)}(X) = f^B(X), f_T^{(2)}(X) \neq f_{T_p}^{(2)}(X) \right\}
\\
&\ +
\mathbb{P}_X\left\{f_{T_p}^{(1)}(X) \neq f^B(X), f_T^{(1)}(X) = f_{T_p}^{(1)}(X) \right\}
-
\mathbb{P}_X\left\{f_{T_p}^{(2)}(X) \neq f^B(X), f_T^{(2)}(X) = f_{T_p}^{(2)}(X) \right\}
\\
=:&\ 
A^{(1)}_n-A^{(2)}_n+B^{(1)}_n-B^{(2)}_n.
\end{align*}
By \eqref{equ:fj_Tp_fB_0}, we have
\begin{align*}
\lim_{n\rightarrow\infty}
\mathbb{P}_X\left\{f_{T_p}^{(j)}(X) \neq f^B(X), f_T^{(j)}(X) = f_{T_p}^{(j)}(X) \right\}
=
0,
\quad
j=1, 2,
\end{align*}
and thus $\lim\limits_{n\rightarrow\infty} B^{(j)}_n =0, j =1, 2$.
By \eqref{equ:fj_Tp_fB_0}, we also have
\begin{align*}
\lim_{n\rightarrow\infty}
\left(
A^{(j)}_n
-
\mathbb{P}_X\left\{f_T^{(j)}(X) \neq f_{T_p}^{(j)}(X) \right\}
\right)
=
0, 
\quad
j=1,2,
\end{align*}
so, by \eqref{equ:ION}, one can show that $A^{(j)}_n$ and $\mathrm{ION}(\mathcal{M}^{(j)}, T)$ share the same limit, $j=1,2$.
Therefore,
\begin{align*}
\lim_{n\rightarrow\infty}
\left[\bigg.
\left(\Big.
\mathbb{P}_X\left\{f_T^{(1)}(X) \neq f^B(X) \right\}
-
\mathbb{P}_X\left\{f_T^{(2)}(X) \neq f^B(X) \right\}
\right)
-
\left(\Big.
\mathrm{ION}(\mathcal{M}^{(1)}, T) - \mathrm{ION}(\mathcal{M}^{(2)}, T)
\right)
\right]
=
0
.
\end{align*}

Because of \eqref{equ:fj_ION}, we have
\begin{align*}
    \limsup_{n\rightarrow\infty}
    \left[\bigg.
    \mathbb{P}_X\left\{f_T^{(1)}(X) \neq f^B(X) \right\} -
    \mathbb{P}_X\left\{f_T^{(2)}(X) \neq f^B(X) \right\}
    \right]
    < 0.
\end{align*}
Further, by Lemma~\ref{lem:f_f_G_Y},
\eqref{equ:fj_te} holds.
\end{proof}

We interpret Theorem~\ref{thm:ION_lower_to_error_lower_limit}.
First,
\eqref{equ:fj_consistency} is a very weak condition.
It assumes the two methods are consistent on the purified training set $T_p$.
In fact, many classical methods have been proved to be consistent.
Furthermore, because there is no noise point in $T_p$, the consistency on $T_p$ is easier to achieve than on $T$.
Even the notoriously easy-to-overfit method, 1NN, is consistent in such a good training set $T_p$ but not necessarily consistent in $T$, according to the Cover-Hart inequality \citep{cover1967_1nn}.

Second, 
\eqref{equ:fj_ION} is about the property of some methods regarding a certain training set.
Instead of subjectively describing the property of the methods,
it measures the influence of the noise points in the particular training set objectively.

Third, according to Theorem~\ref{thm:ION_lower_to_error_lower_limit}, 
the decrease of ION implies that the method is minimizing the influence of the noise points and thus enhancing the generalization ability.
It means that, for most methods,
the ION is a good indicator of the test error.

Fourth, there is no natural conflict between interpolation and lower ION. For a classifier, the purpose of interpolating is to take all information contained in the signal points as much as possible, while the goal of lowering ION is to reduce the impact attributed to the noise points. 

In order to have a concrete understanding of Proposition~\ref{prop:ION_lower_to_error_lower} and Theorem \ref{thm:ION_lower_to_error_lower_limit},
we give a 2-dim toy example. First, the population is 
\begin{align*}
    \mathbb{P}\{Y=1|X=\bm{x}\} = 
    \begin{cases} 
    0.1, & \text{ if } x_1 < 0,\\
    0.9, & \text{ if } x_1 \geq 0,
    \end{cases}
\end{align*}
where $X$ is uniformly distributed in $(-1, 1]^2$.
In other words, only the first dimension of $X$ is relevant to $Y$, while the second dimension contributes no information.
One can easily solve the Bayes classifier $f^B(\bm{x})=\mathrm{sign}\{x_1\}$ and the Bayes error $q = 0.1$.

Second, we randomly generate a training set $T$ with a size $n=500$, as in Fig. \ref{fig:drcT}.
The training set $T$ is composed of $460$ signal points and $40$ noise points.
Roughly speaking, the yellow triangles to the left and the blue circles to the right are all noise points.
Particularly, on the left and bottom side of the graph in Fig. \ref{fig:drcT}, there is a solid triangle $\bm{x}_{i_0} = (-0.79, -0.41)$, which is the noise point that would be discussed later.

\begin{figure}[htpb]
\centering
\includegraphics[width=0.45\textwidth]{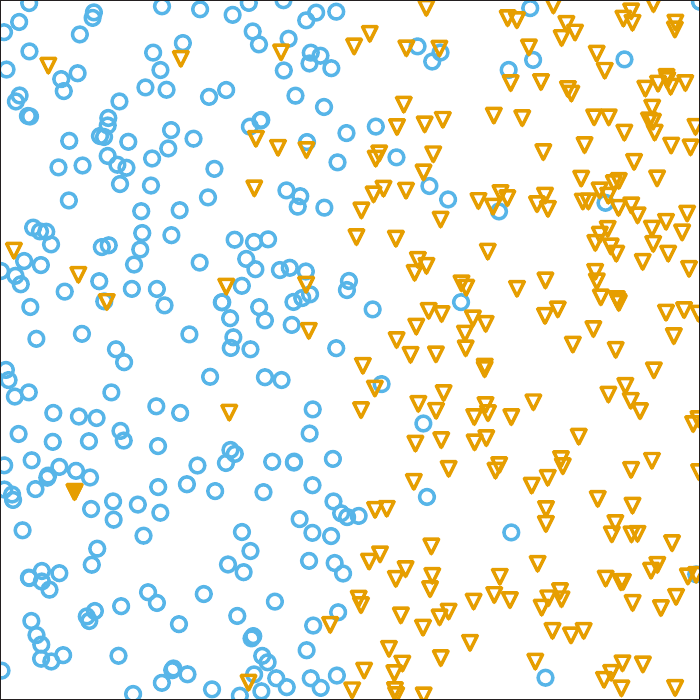}
\caption{The training set $T$.\label{fig:drcT}}
\end{figure}

Third, we apply two methods, $\mathcal{M}^{(1)}$ (1NN) and $\mathcal{M}^{(2)}$ (AdaBoost\footnote{To be more specific, for the AdaBoost we used, the base learners are decision trees with a maximum depth 4, and the number of iterations is 50.}), to generate interpolating classifiers on the training set $T$ respectively. Fig. \ref{fig:drcTT}(\subref{subfig:drc1NN}) is the classifier $f^\text{1NN}_T$ generated by method 1NN, while \ref{fig:drcTT}(\subref{subfig:drcAda}) $f^\text{AdaBoost}_T$.
The purple vertical dotted line ($x_1=0$) is the watershed of $f^B$, while the black solid lines are the decision boundaries of $f^{\mathcal{M}^{(\cdot)}}_T$.
Both classifiers are interpolating, which means that $f_T^{\mathcal{M}^{(\cdot)}}(\bm{x}_i) = y_i$, $\forall i$, even though they are from different methods.

\begin{figure}[htpb]
\centering
\subcaptionbox{$f^\text{1NN}_T$\label{subfig:drc1NN}}
{\includegraphics[width=0.45\textwidth]{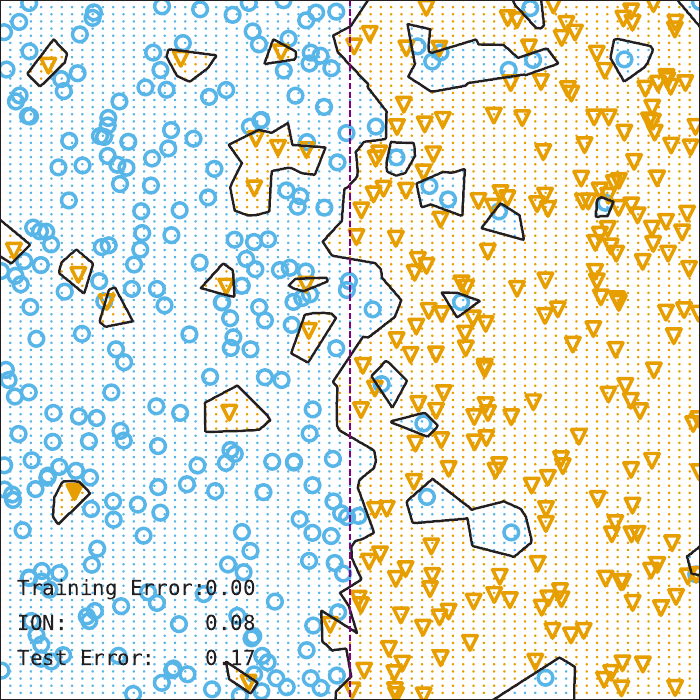}}
\subcaptionbox{$f^\text{AdaBoost}_T$\label{subfig:drcAda}}
{\includegraphics[width=0.45\textwidth]{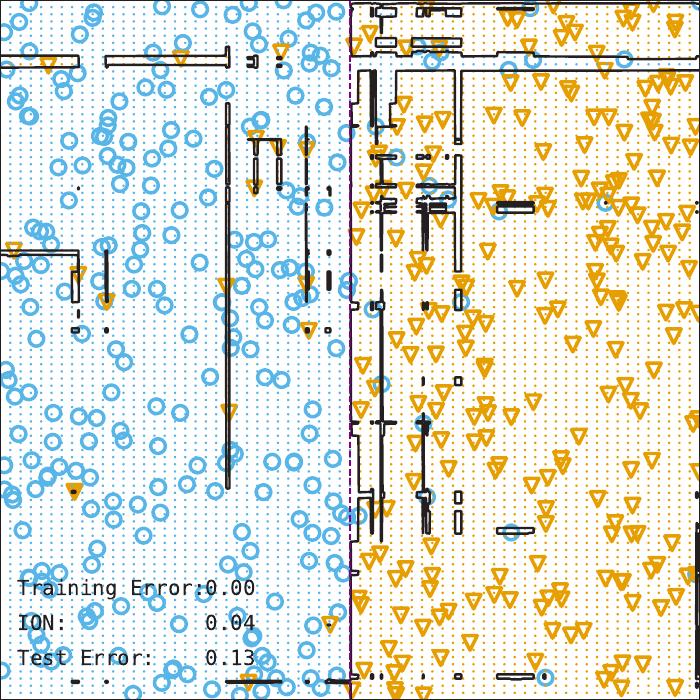}}
\caption{The training set $T$ and the classifiers: AdaBoost has lower ION than 1NN.\label{fig:drcTT}}
\end{figure}

Fourth, we argue that the ION of AdaBoost is lower than that of 1NN on the training set $T$.
the classifiers in Fig. \ref{fig:drcTT} are different: The decision boundary of 1NN in Fig. \ref{fig:drcTT}(\subref{subfig:drc1NN}) is smooth and natural, while that of AdaBoost in Fig. \ref{fig:drcTT}(\subref{subfig:drcAda}) is sharp and uneven.
However, we argue that the sharp and uneven is better than the smooth and natural in the sense of minimizing and localizing the influence of the noise points.
For the isolated noise points, 
the regions surround them in \ref{fig:drcTT}(\subref{subfig:drcAda}) are smaller and narrower than that in \ref{fig:drcTT}(\subref{subfig:drc1NN}).
In detail, we focus on the particular noise point lies at about $\bm{x}_{i_0} = (-0.79, -0.41)$, which is the solid triangle.
The area around $\bm{x}_{i_0}$ of $f^\text{AdaBoost}_T$ in \ref{fig:drcTT}(\subref{subfig:drcAda}) is very small,
while that of $f^\text{1NN}_T$ in \ref{fig:drcTT}(\subref{subfig:drc1NN}) is a big irregular polygon---the influence of the noise point $\bm{x}_{i_0}$ seems to be lower for AdaBoost than 1NN.\footnote{It is noteworthy that the influence of the noise points are acting jointly rather than individually, but it does not matter in this heuristic case.}

Last, we calculate the ION and the test error, summarized in Table~\ref{tab:toy example}, where $\mathrm{ION}(\text{1NN}, T) \approx 0.08$ and $\mathrm{ION}(\text{AdaBoost}, T) \approx 0.04$, 
$\mathbb{P}\left\{f_T^{\text{1NN}}(X) \neq Y \right\} \approx 0.17$ and 
$\mathbb{P}\left\{f_T^{\text{AdaBoost}}(X) \neq Y \right\} \approx 0.13$.
We can observe that the results are in line with our theorem, i.e., the lower the ION, the lower the test error.

\begin{table}[htbp]
  \centering
  \caption{The ION and the test error.}
    \begin{tabularx}{.9\textwidth}{Y|Y|Y|Y}
    \hline
    Method & ION & Training Error & Test Error \\
    \hline
    $f^B$ & - & - & 0.10 \\
    \hline
    $\mathcal{M}^{(1)}=\mathrm{1NN}$ & 0.08 & 0 & 0.17 \\
    \hline
    $\mathcal{M}^{(2)}=\mathrm{AdaBoost}$ & 0.04 & 0 & 0.13 \\
    \hline
    \end{tabularx}
  \label{tab:toy example}
\end{table}

Overall, this section connects the ION and the test error.
Both the theoretical derivation and the toy example of simulation demonstrate the importance of ION.
Particularly, the toy example explains why 1NN is easy to overfit, while AdaBoost not.
However, AdaBoost is only a general term for a class of methods, since both the base learners and the number of iterations need to be specified. By choosing different kinds of base learners and different numbers of iterations, we can generate a tremendous amount of specific methods. 
In Section \ref{subsection: ION and AdaBoost}, we take a close look at the performance of AdaBoost with different hyperparameters from our new perspective: ION.

\subsection{The ION and AdaBoost}
\label{subsection: ION and AdaBoost}
AdaBoost mainly has two hyperparameters.
One of them is the complexity of the base learners.
The decision trees are one of the most popular base learners of AdaBoost, which is the classical base learner in the monograph \citet{book:esl}.
In this paper, we use the maximum depth of decision trees to indicate the complexity of the base learners. The deeper the decision trees, the more complex the base learners, and the more complex the AdaBoost.
The other is the number of iterations, which is the number of the base learners added in total.
The higher the number of iterations, the more complex the AdaBoost.

This section corroborates the conclusion of \citet{wyner2017AdaBoost_RandomForest} with our newly defined concept ION. 
We show their conclusion that AdaBoost based on large decision trees without early stopping is better. We want to show that AdaBoost generates interpolating classifiers, and both the ION and the test error decrease as the depth of the base learners and the number of iterations increase. Instead of comparing AdaBoost with 1NN, we digest AdaBoost itself with different parameters in details via high-dimensional population of simulation.

The simulation population is
\begin{align*}
    \mathbb{P}\{Y=1|X=\bm{x}\} = 
    \begin{cases} 
    0.1, & \text{ if } x_1\cdot x_2 \cdot x_3 < 0,\\
    0.9, & \text{ if } x_1\cdot x_2 \cdot x_3 \geq 0,
    \end{cases}
\end{align*}
where
$X$ is uniformly distributed in $(-1, 1]^6$.
We randomly generate a training set $T$ with $n=500$, 
and compare the results of AdaBoost with different hyperparameters.

In order to explain the reason why AdaBoost without early stopping might be better, we compare the results of AdaBoost with different numbers of iterations $m = 1, 2,\dots, 250$ but the same maximum depth of the base decision trees. The maximum depth is set as 5. 
We denote the corresponding classifiers by $f_T^{(m)}, m=1,2, \dots, 250$.

The results are in Fig. \ref{fig:Ada_M}.
The $x$-axis in the figure represents the number of iterations $m$. 
We clarify the three lines in detail.
The green dashed line is the training error of $f^{(m)}_T$ on its training set $T$,
the red dashed-dotted line is the test error of $f^{(m)}_T$,
and the blue solid line is the ION of AdaBoost in the training set: $\mathrm{ION}(\mathcal{M}^{(m)}, T)$.
All the three lines decrease sharply when $n < 20$.
When $n \geq 20$, the training error remains $0$, but the test error and ION keep decreasing.

\begin{figure}[htpb]
\centering
\includegraphics[width=.9\textwidth]{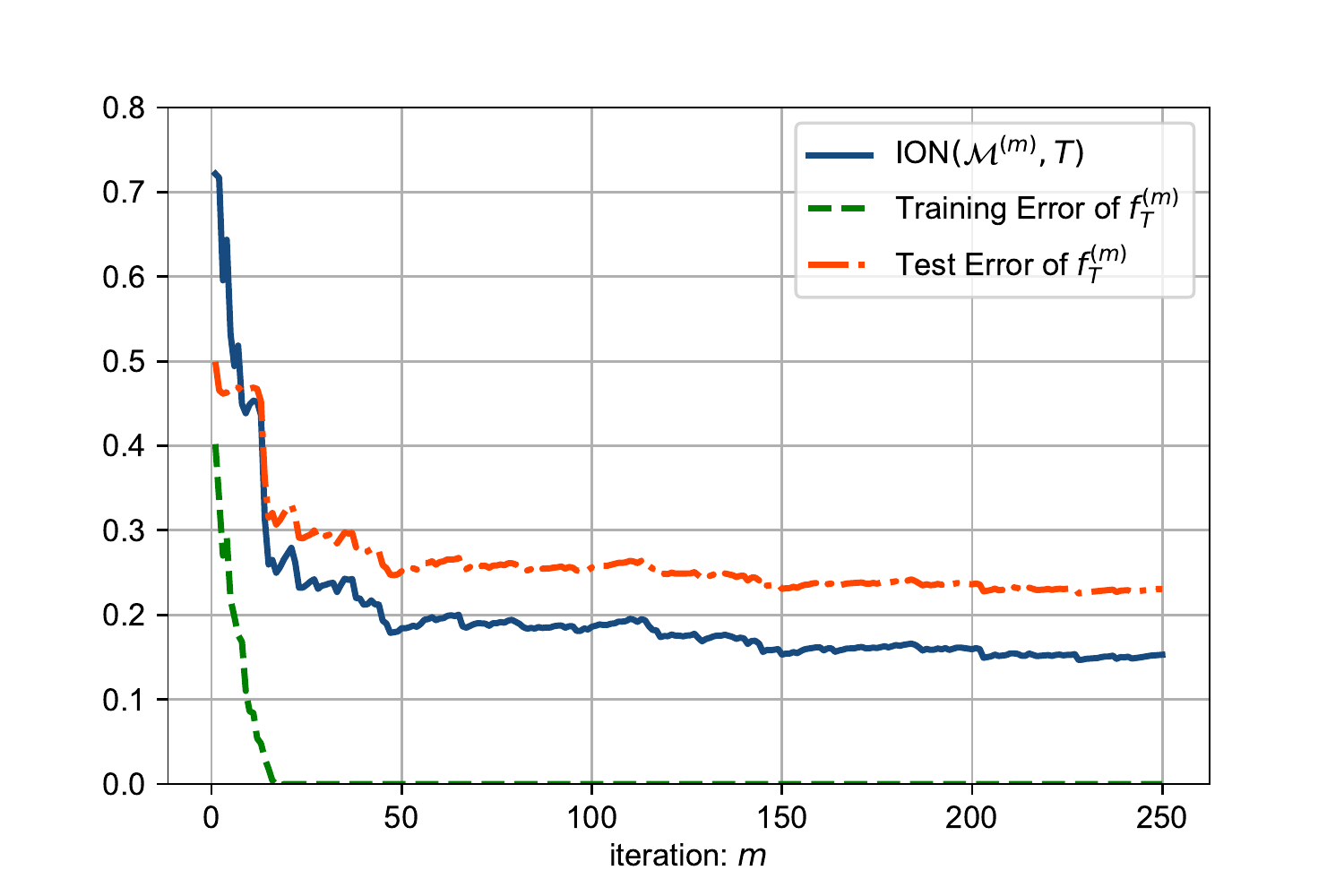}
\caption{The performance of AdaBoost regarding to $m=1,2,\dots,250$. \label{fig:Ada_M}}
\end{figure}

From Fig. \ref{fig:Ada_M}, we have the following observations.
First,
AdaBoost is minimizing the influence of the noise points.
When $m \geq 20$, the test error decreases but the training error remains $0$.
A natural question is, what is AdaBoost doing?
There are many explanations.
\citet{wyner2017AdaBoost_RandomForest} believed that AdaBoost is self-averaging and generating a ``spiked-smooth'' classifier by minimizing the the influence of the noise points.
We corroborate their work with the blue solid line ION.
When $m \geq 20$, 
although the training error remains $0$, 
the ION continues to decrease,
which reflects the decrease of the influence of the noise points.
Thus, as the number of iterations increases, AdaBoost keeps interpolating, and simultaneously minimizes the influence of the noise points.
Second,
the iteration of AdaBoost can be divided into two stages:
The first stage is the sharp decrease of the training error ($m < 20$),
and the second stage is the decrease of ION ($m\geq 20$).
The first stage can be considered as the formation of the rough skeleton of the classifier,
while the second stage can be treated as the process of the details with the minimization of the influence of the noise points.\footnote{However, the two stages can not be divided arbitrarily, because ION may also play a role in the first stage.}

For another, 
we show that AdaBoost based on deep/large decision trees is better, and explain it by ION.
Specifically, we apply AdaBoost based on different decision trees but the same number of iterations $m=250$,
where the base decision trees have different maximum depths from $1$ to $8$.
In other words, the number of the terminal leaves of the base decision trees varies from $2$ to $256$.
We denote the corresponding classifiers by $f^{(j)}_T, j=1,2, \dots, 8$.

The results are presented in Fig. \ref{fig:Ada_depth}. 
The $x$-axis in the figure represents the maximum depth, i.e., $j$ for $f^{(j)}_T$.
The three lines are the same as those in Fig. \ref{fig:Ada_M}, and so are the interpretations.

\begin{figure}[htpb]
\centering
\includegraphics[width=.9\textwidth]{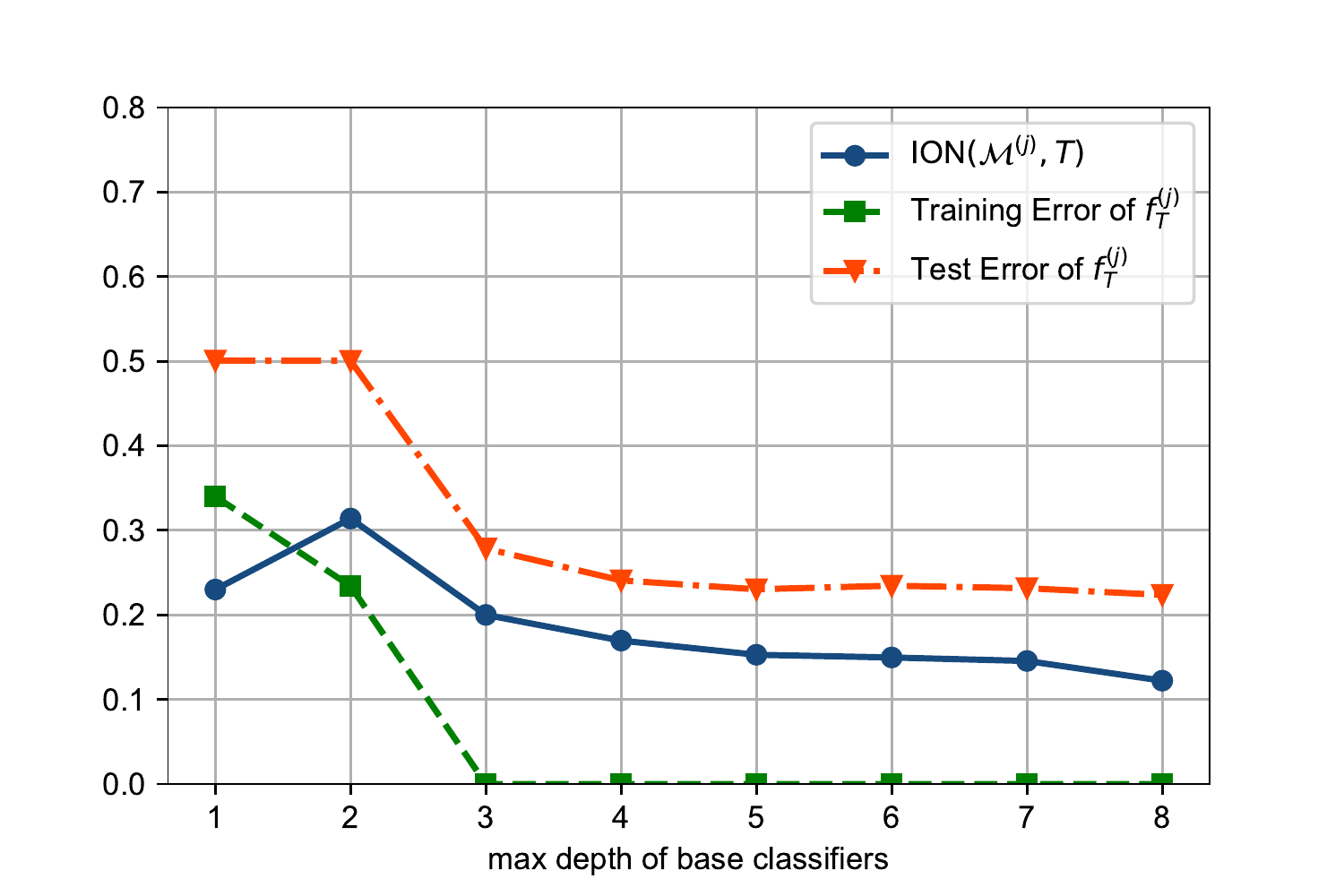}
\caption{The performance of AdaBoost regarding to $j=1,2,\dots,8$.\label{fig:Ada_depth}}
\end{figure}

Overall, AdaBoost based on large decision trees without early stopping is better, 
which can be explained as the decrease of the ION.
Given the condition that the training error is $0$,
the influence of the noise points decreases as the depth of the base decision trees and the number of iterations increase.

Now, return to the main line of the paper, 
we show that AdaBoost would not overfit even interpolating, 
when digging out complex structures of factors in constructing equal-weighted portfolios.
As it was emphasized by \citet[p. 15]{de2018ml}, 
the linear methods are awfully simplistic and useless, and would ``fail to recognize the complexity of the data''.
The academia and industry shift their focus to the non-linear ones.
There are a tremendous amount of machine learning methods applied in various data and fields in finance.
Many of the machine learning methods suffer from the overfitting and the low interpretation.
However, AdaBoost is not heavily affected by them as illustrated above.

In Section~\ref{sec:empirical}, 
we give empirical studies about specific factors or strategies, and prove the advantage of AdaBoost in portfolio management. But we want to clarify what kind of non-linear information can AdaBoost dig out first. 

%% file: secs/limit_of_adaboost.tex
\section{Base learners of AdaBoost}\label{sec:XOR_counter}
AdaBoost is a boosting method. It boosts the performance of a series of base learners, or ``weak classifiers''. People usually choose shallow trees (such as ``stumps'', i.e., decision trees with only one layer) as base learners since they are ``weak'' enough and thus can avoid overfitting.

However, in many fields, especially in the area of portfolio management, using stumps as base learners may not capture the nature of the population, since the population is usually rather complicated. \cite{wyner2017AdaBoost_RandomForest} proposed that the deep and large trees will allow the base learners to interpolate the data without overfitting, and it is a better choice to use deep trees as base learners. We have already shown the result mathematically in Section \ref{subsection: ION and AdaBoost} from the perspective of the ION. 

In this section, we discuss the shortcomings of AdaBoost based on stumps. We first show that stumps cannot deal with the ``XOR'' classification problem. Then, we generalize the result and demonstrate that AdaBoost based on stumps cannot deal with populations without ``comonotonicity''. These kinds of populations are common in finance, since the investment activities in the financial world are usually rather complicated and interactive.

\subsection{The ``XOR'' population}
\label{subsection: XOR}
In this section, we use a toy example to show that the shallow trees (especially stumps) are not always capable to capture the patterns of the population. We introduce the Boolean operator ``exclusive OR'' (XOR) first. 

\begin{defn}[2-XOR]
The 2-XOR function is defined as $\mathrm{XOR}_2:\mathbb{R}^2 \to \{\pm1\}$ such that
$$
\mathrm{XOR}_2(x_1, x_2) =
\begin{cases}
  -1, & \mbox{if } x_1\cdot x_2 \geq 0, \\
  1, & \mbox{if } x_1\cdot x_2 < 0.
\end{cases}
$$
\end{defn}

\begin{defn}[$k$-XOR]  For $k>2$, the $k$-XOR function, denoted as $\mathrm{XOR}_k$, is defined recursively as
$$
\mathrm{XOR}_{k}(x_1, x_2, \dots, x_k)=\mathrm{XOR}_2\left(\mathrm{XOR}_{k-1}(x_1, x_2, \dots, x_{k-1}),x_k\right).
$$
\end{defn}

\begin{figure}[htpb]
\centering
\subcaptionbox{2-XOR}
{\includegraphics[width=0.48\textwidth]{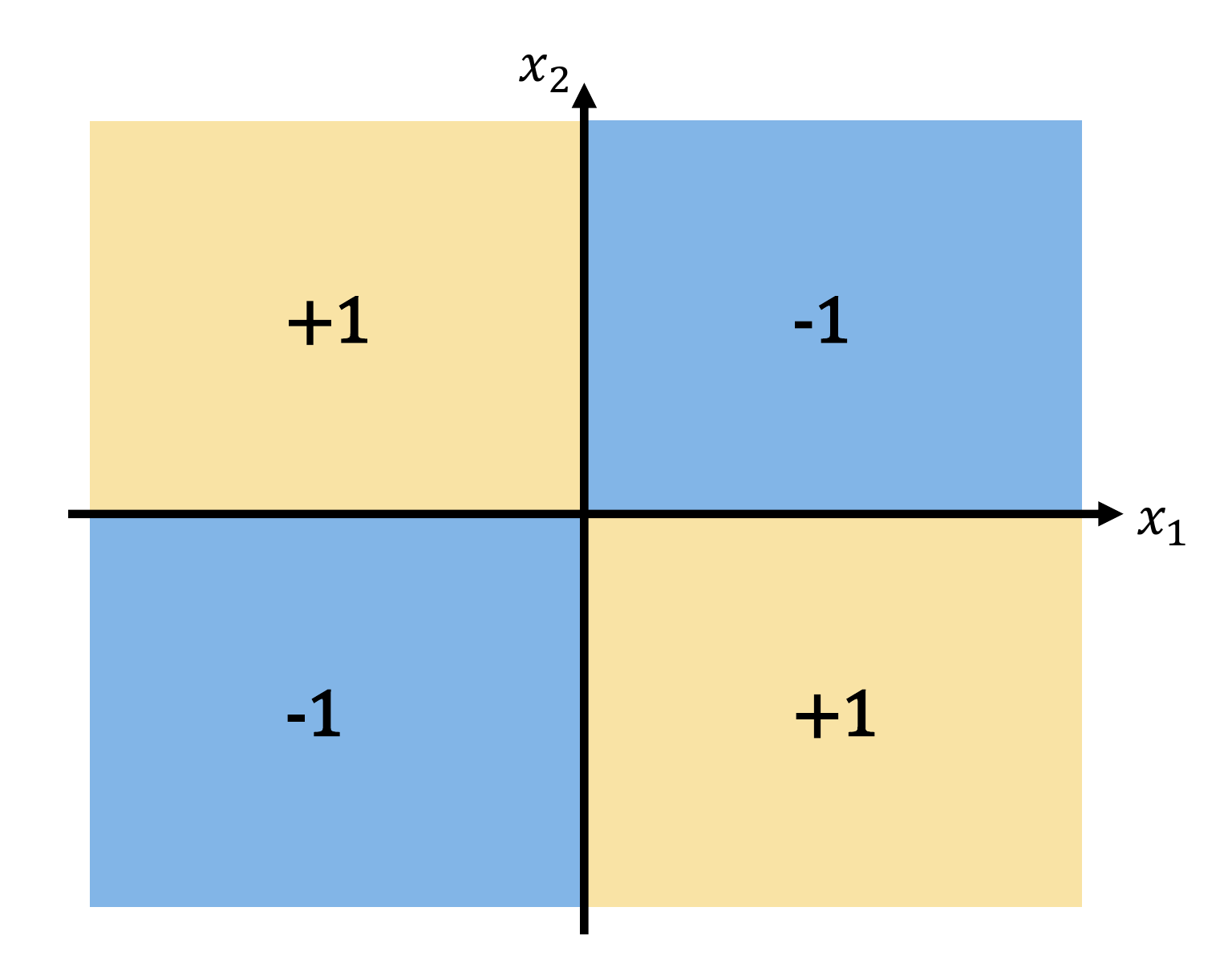}}
\subcaptionbox{3-XOR}
{\includegraphics[width=0.48\textwidth]{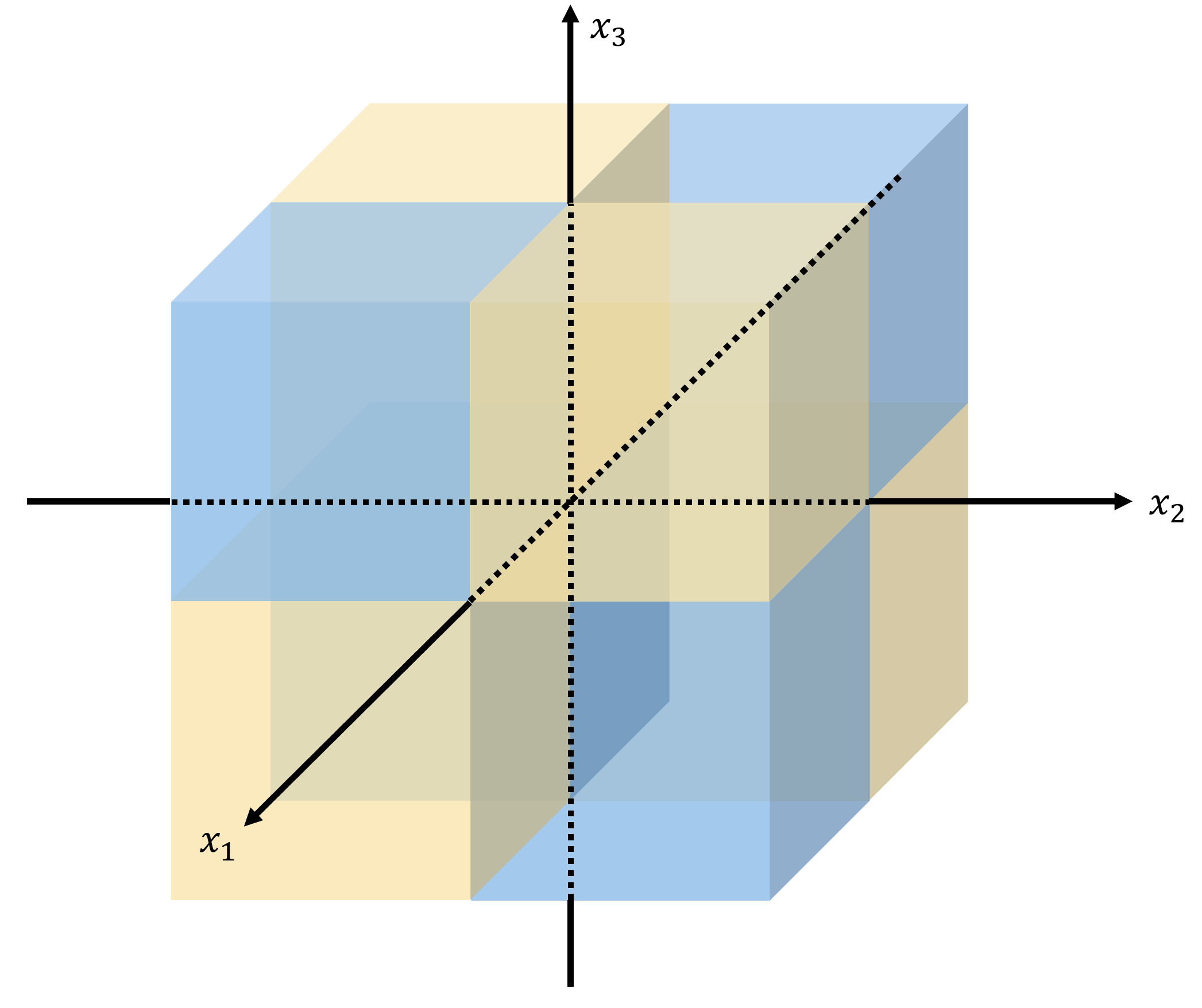}}
\caption{\label{fig:intuition_of_xor} Intuitions of 2-XOR and 3-XOR functions.}
\end{figure}

The Boolean operator $k$-XOR is an important function in computer science (for instance, the parity check), and it is also a classical example in \cite{book:esl}. It can also bring insights into portfolio management, since the $k$-XOR can characterize the interaction among different factors. There are many studies focusing on the interaction among various factors \citep{asness2018size}. Fig. \ref{fig:intuition_of_xor} are intuitive illustrations of the 2-XOR and the 3-XOR functions. The outputs are not the same in adjacent quadrants (or octants), which is a common pattern of interaction.

Now we show that the stumps cannot deal with the classification problems with the Bayes classifier $f^B = \mathrm{XOR}_k$, even in the case that the Bayes error is 0. For instance, if we use a stump classifier $f$ to classify the 2-XOR function, we can easily show that $\mathbb{P}_{x_1, x_2}\left\{f(x_1,x_2)\neq \mathrm{XOR}_2(x_1,x_2)\right\}$ is always $50\%$ no matter how the stump is trained. That is because, a stump is equivalent to a partition of $\mathbb{R }^2$ along the direction of one axis. After the partition, both half-spaces still contain values of 1 (accounts for $50\%$) and $-1$ (50\%), which leads to a test error $50\%$. 

The conclusion can also be generalized to high-dimensional spaces. Let $f^{(\leq k)}$ denote a decision tree whose depth is no more than $k$, and $f^{(k)}$ denote a decision tree whose depth is $k$. We have the following result in Theorem \ref{thm:xor}. 
\begin{thm} \label{thm:xor}
Applying a decision tree $f^{(\leq k)}$ on the $(k+1)$-$\mathrm{XOR}$ classification problem will always lead to $\mathbb{P}_X\left\{f^{(\leq k)}(X) = \mathrm{XOR}_{k+1}(X)\right\}=50\%$, where $X=(x_1,\dots,x_{k+1})$.
\end{thm}

\begin{proof}
We prove the theorem by induction. The case of $k=1$ has already been proved. Now we assume that our conclusion holds for $f^{(\leq k-1)}$. We want to prove that it also holds for $f^{(k)}$.

Without loss of generality, we assume that the splitting variable of $f^{(k)}$'s first layer is the 1st feature $x_1$, then $$f^{(k)}=\mathbf{1}_{\{x_1\leq c\}}f_1^{(k-1)}+\mathbf{1}_{\{x_1> c\}}f_2^{(k-1)},$$ where $f_1^{(k-1)}$ and $f_2^{(k-1)}$ represent the left subtree and the right subtree of $f^{(k)}$'s top node respectively, and $c$ is the splitting value. Let $X_{[-1]}=(x_2,\dots,x_{k+1})$, then
\begin{eqnarray*}
 \mathbb{P}_X\left\{f^{(k)}(X) = \mathrm{XOR}_{k+1}(X)\right\} 
&=& \mathbb{P}_X\left\{f_1^{(k-1)}(X)=\mathrm{XOR}_{k+1}(X)\Big|x_1\leq c\right\}\mathbb{P}_X\{x_1\leq c\}\\
&& +\mathbb{P}_X\left\{f_2^{(k-1)}(X)=\mathrm{XOR}_{k+1}(X)\Big|x_1>c\right\}\mathbb{P}_X\{x_1>c\}.
\end{eqnarray*}
Assuming $c>0$ without loss of generality, then we have
\begin{eqnarray*}
&&\mathbb{P}_X\left\{f_1^{(k-1)}(X)=\mathrm{XOR}_{k+1}(X)\Big|x_1\leq c\right\}\\
&=&\mathbb{P}_X\left\{f_1^{(k-1)}(X)=\mathrm{XOR}_{k+1}(X)\Big|x_1\leq c, x_1\leq 0\right\}\mathbb{P}_X\{x_1\leq0\}\\
&& +\mathbb{P}_X\left\{f_1^{(k-1)}(X)=\mathrm{XOR}_{k+1}(X)\Big|0<x_1\leq c\right\}\mathbb{P}_X\{x_1>0\}\\
&=&\mathbb{P}_X\left\{f_1^{(k-1)}(X_{[-1]})=\mathrm{XOR}_{k}(X_{[-1]})\right\}\mathbb{P}_X\{x_1\leq0\}\\
&&+\mathbb{P}_X\left\{-f_1^{(k-1)}(X_{[-1]})=\mathrm{XOR}_{k}(X_{[-1]})\right\}\mathbb{P}_X\{x_1>0\},
\end{eqnarray*}
and 
$$
\mathbb{P}_X\left\{f_2^{(k-1)}(X)=\mathrm{XOR}_{k+1}(X)\Big|x_1>c\right\}=\mathbb{P}_X\left\{-f_2^{(k-1)}(X_{[-1]})=\mathrm{XOR}_{k}(X_{[-1]})\right\}.
$$
Our inductive assumption tells us that, for the $\mathrm{XOR}_{k}$ classification problem, both $f_1^{(k-1)}$ and $f_2^{(k-1)}$ will have a $50\%$ error. Hence, the three probabilities $\mathbb{P}_X\left\{f_1^{(k-1)}(X_{[-1]})=\mathrm{XOR}_{k}(X_{[-1]})\right\}$, $\mathbb{P}_X \left\{-f_1^{(k-1)}(X_{[-1]}) = \mathrm{XOR}_{k}(X_{[-1]})\right\}$ and $\mathbb{P}_X\left\{-f_2^{(k-1)}(X_{[-1]})=\mathrm{XOR}_{k}(X_{[-1]})\right\}$ are all equal to $50\%$. Finally, 
\begin{eqnarray*}
& &\mathbb{P}_X\left\{f^{(k)}(X) = \mathrm{XOR}_{k+1}(X)\right\}\\
&=&\left[50\%\mathbb{P}_X\{x_1\leq0\}+50\%\mathbb{P}_X\{x_1>0\}\right]\mathbb{P}_X\{x_1\leq c\}+50\%\mathbb{P}_X\{x_1>c\}\\
&=&50\%\mathbb{P}_X\{x_1\leq c\}+50\%\mathbb{P}_X\{x_1>c\}\\
&=&50\%.
\end{eqnarray*}
\end{proof}

In the proof above, we suppose that each component of $X$ would be split just only one time. In other words, once the CART algorithm \citep[p. 305]{book:esl} split a decision tree at $x_1$, it will not split at $x_1$ again in other layers. It is just for clarity and conciseness, because one can use the total probability formula to deal with the more complicated situations.

Although the $k$-XOR is a special case that each factor interacts with other factors, it is enough to demonstrate that shallow decision trees (especially for one-layer stumps) may be unable to deal with factors that are not independent of each others. 

\subsection{The population without ``comonotonicity''}\label{ringanddiag}

In this section, we show the shortcomings of AdaBoost based on stumps by introducing the concept of ``comonotonicity''. The conclusion in this section can be regarded as an extension of Section \ref{subsection: XOR}, since the XOR function do not have the property of comonotonicity, as we will discuss later. 

\begin{defn}[Comonotonic Population]
\label{def:comonotonicity}
A population $X \in \mathcal{X}=\mathbb{R}^k, Y \in \mathcal{Y}=\{\pm1\}$ is comonotonic, if its Bayes classifier $f^B$ satisfies: for any constant $c$ and any $i=1,\dots,k$, there exists an $\varepsilon>0$ such that for each $a\in(c-\varepsilon,c)$ and $b\in(c,c+\varepsilon)$, the elements in 
$$
\left\{ f^B(x_1, \dots, x_{i-1}, a, x_{i+1}, \dots, x_k)-f^B(x_1, \dots, x_{i-1}, b, x_{i+1}, \dots, x_k): X_{[-i]} \in \mathbb{R}^{k-1} \right\}
$$
are all non-positive or all non-negative, where $X_{[-i]}=(x_1,\dots,x_{i-1},x_{i+1},\dots,x_k)$.
\end{defn}

To give an intuition of comonotonicity, in Fig. \ref{fig:not_comonotonic}, we give three examples of populations which are not comonotonic. For Fig. \ref{fig:not_comonotonic}(\subref{subfig:XOR_arrow}), \ref{fig:not_comonotonic}(\subref{subfig:ring_arrow}) and \ref{fig:not_comonotonic}(\subref{subfig:diag_arrow}), the decision boundaries of their Bayes classifiers form shapes of an XOR, a ring, and a diagonal band, respectively. The yellow (light) region takes a value of $+1$, and the blue (dark) region $-1$. Note that, in each figure, there both exist arrows from values of $-1$ to $+1$, and arrows from values of $+1$ to $-1$. These arrows tell us that the populations are not comonotonic.

\begin{figure}[htpb]
\centering
\subcaptionbox{\label{subfig:XOR_arrow}XOR}
{\includegraphics[width=0.32\textwidth]{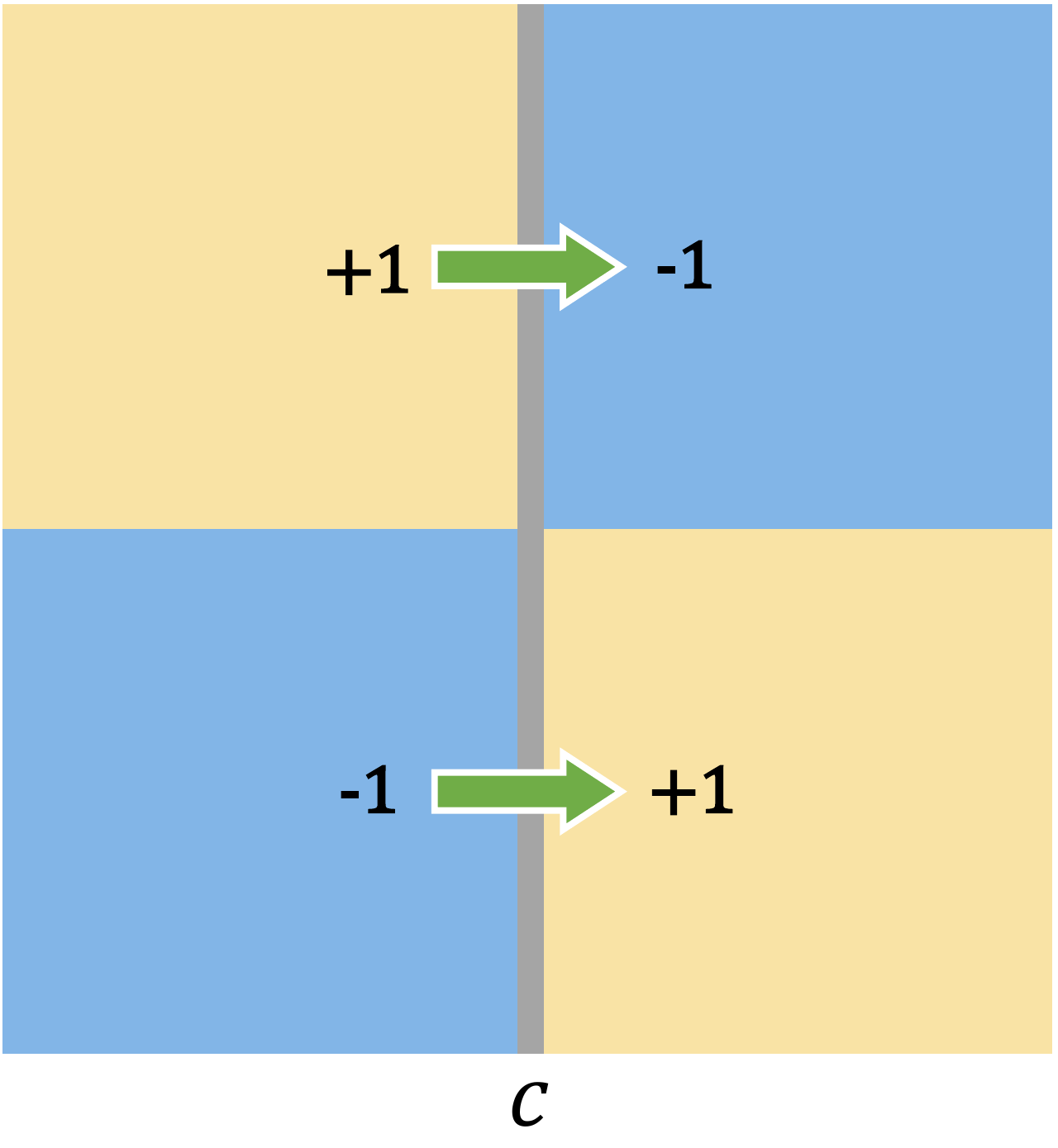}}
\subcaptionbox{\label{subfig:ring_arrow}Ring}
{\includegraphics[width=0.32\textwidth]{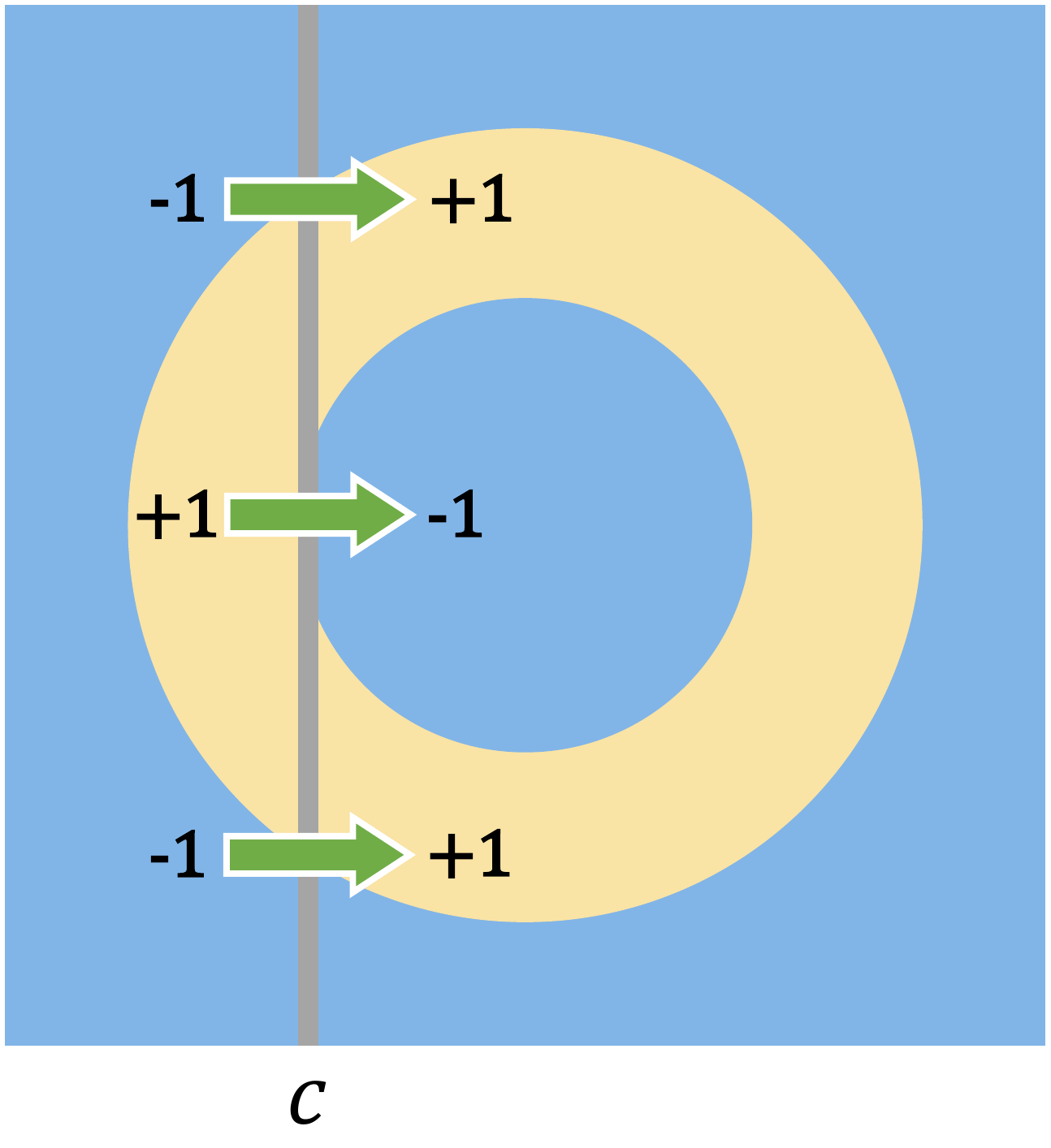}}
\subcaptionbox{\label{subfig:diag_arrow}Diagonal}
{\includegraphics[width=0.32\textwidth]{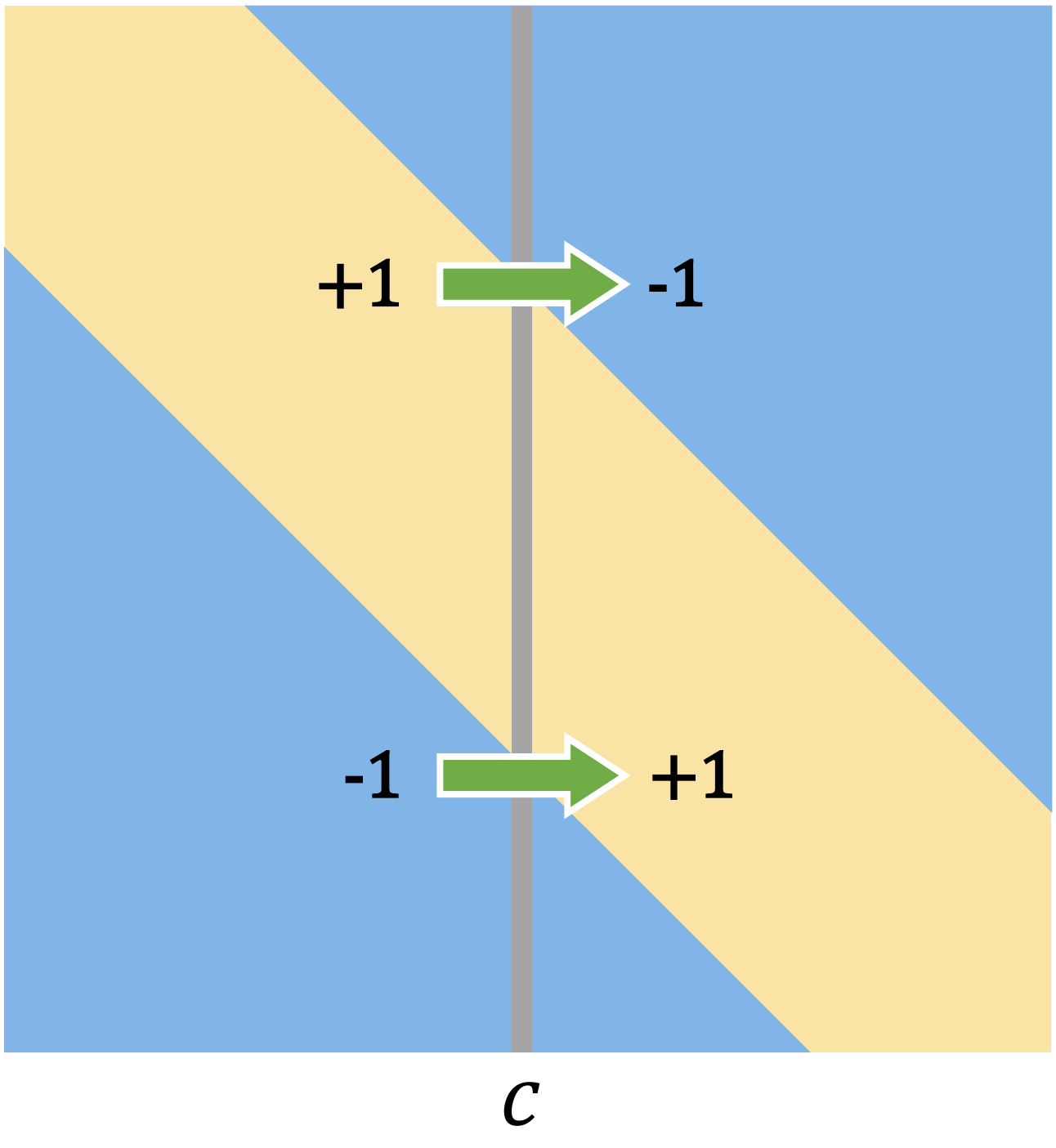}}
\caption{\label{fig:not_comonotonic} The populations WITHOUT comonotonicity.}
\end{figure}

The following theorem shows that AdaBoost based on stumps cannot deal with populations without comonotonicity.

\begin{thm}
For a population in $\mathbb{R}^k$ with a Bayes classifier $f^B(x_1,x_2,\dots, x_k)$, a necessary condition for the classifier trained by AdaBoost based on stumps can converge to $f^B$ as the number of iterations $m\to\infty$ is: the population is comonotonic. 
\end{thm}

\begin{proof}
The AdaBoost.M1 algorithm in  \citet[p. 339]{book:esl} shows that, if the number of iterations is $m$, the final strong classifier $f^{(m)}$ must takes the form 
$$f^{(m)}(x_1,x_2,\dots,x_k)=\mathrm{sign}\left[\sum_{s=1}^{m}\alpha_s g_s (x_1,x_2,\dots,x_k)\right],$$
where $g_s(\cdot), s=1,\dots,m$ are base learners (stumps). In other words, $f^{(s)}(\cdot)$ must be a linear combination of base learners. 

A stump ``$\mathrm{A}$'' with $k$ variables can be expressed as
$$
\textrm{stump}_\textrm{A}(x_1, x_2,\dots,x_k) = \mathrm{sign}(x_i\bowtie c):=
\begin{cases}
  1, & x_i \bowtie c, \\
  -1, & \mbox{otherwise},
\end{cases}
$$
where $\bowtie \in \{\geq, \leq, >,<\}$, and the splitting variable of the stump is the $i$-th feature. 

Without loss of generality, we require that $\bowtie$ can only be $\leq$ or $>$.
Then the linear combination of stumps trained by AdaBoost can be represented as
$$
f^{(m)}(x_1,x_2,\dots,x_k)=\mathrm{sign}\left[L +\sum_{s=1}^{m}\alpha_s \mathrm{sign}(x_{i_s}\bowtie c_s)\right].
$$
where $L$ is a constant, $x_{i_s}$ is the splitting variable of the $s$-th stump, and $c_s$ is the splitting value of the $s$-th stump. 
Since $\mathrm{sign}(x_{i}>c) = -\mathrm{sign}(x_{i}\leq c) $, we can adjust all inequality signs to the same direction:
$$
f^{(m)}(x_1,x_2,\dots,x_k)=\mathrm{sign}\left[L + \sum_{s=1}^{m}\alpha_s \mathrm{sign}(x_{i_s}\leq c_s)\right].
$$

For simplicity, let us consider the 2-dim case ($k=2$). One can generalize the following conclusions to high-dimensional spaces similarly. According to the splitting variable of each stump, we can separate the $m$ stumps into two groups as 
$$
f^{(m)}(x_1,x_2)=\mathrm{sign}\left[L + \left(\sum_{i=1}^{m_1}\beta_i \mathrm{sign}(x_1\leq a_i)\right) + \left(\sum_{j=1}^{m_2}\gamma_j \mathrm{sign}(x_2\leq b_j)\right)\right],
$$
where $m_1$ is the number of stumps with $x_1$ as the splitting variable, $m_2$ is the number of stumps with $x_2$ as the splitting variable, and $m_1+m_2=m$. Without loss of generality, we assume that $a_1\leq a_2 \leq \dots \leq a_{m_1}$, and $b_1 \leq b_2 \leq \dots \leq b_{m_2}$.

Recall the definition of comonotonicity (Definition \ref{def:comonotonicity}). For any constant $c$, take any $\varepsilon>0$, $a\in(c-\varepsilon,c)$ and $b\in(c,c+\varepsilon)$. We sort $a,b,c$ and $a_1,\dots,a_{m_1}$ together as
$$
a_{m_{1,a}} \leq a \leq a_{m_{1,a}+1} \leq \dots \leq a_{m_{1,c}} \leq c \leq a_{m_{1,c}+1} \leq \dots \leq a_{m_{1,b}} \leq b \leq a_{m_{1,b}+1}.
$$
Then, from the expression of $f^{(m)}(x_1,x_2)$, we have
$$
f^{(m)}(a,x_2)-f^{(m)}(b,x_2)\equiv\sum_{i=m_{1,b}}^{m_{1,a}}\beta_i, \quad\forall x_2,
$$
i.e., it does not depend on $x_2$. Similarly, we also have that $f^{(m)}(x_1,a)-f^{(m)}(x_1,b)$ does not depend on $x_1$. Let $m\to\infty$, and if the algorithm will converge, then
$$
\lim_{m\to\infty} \left[ f^{(m)}(a,x_2)-f^{(m)}(b,x_2) \right]
$$
and
$$
\lim_{m\to\infty} \left[ f^{(m)}(x_1,a)-f^{(m)}(x_1,b) \right]
$$
will also be constants which do not depend on $x_2$ and $x_1$, respectively. Therefore, according to the definition of comonotonicity, $f^{(m)}$ cannot converge to $f^B$ if the population is not comonotonic. 
\end{proof}

To show the intuition of $f^{(m)}$ in the proof above, let $f(x_1,x_2)=L + \left(\sum\limits_{i=1}^{m_1}\beta_i \mathrm{sign}(x_1\leq a_i)\right) + \left(\sum\limits_{j=1}^{m_2}\gamma_j \mathrm{sign}(x_2\leq b_j)\right)$. Fig. \ref{fig:explanation_of_f} illustrates the property of the final strong classifier $f$.

\begin{figure}[htpb]
\centering
\subcaptionbox{\label{subfig:explanation_3d}}
{\includegraphics[width=0.45\textwidth]{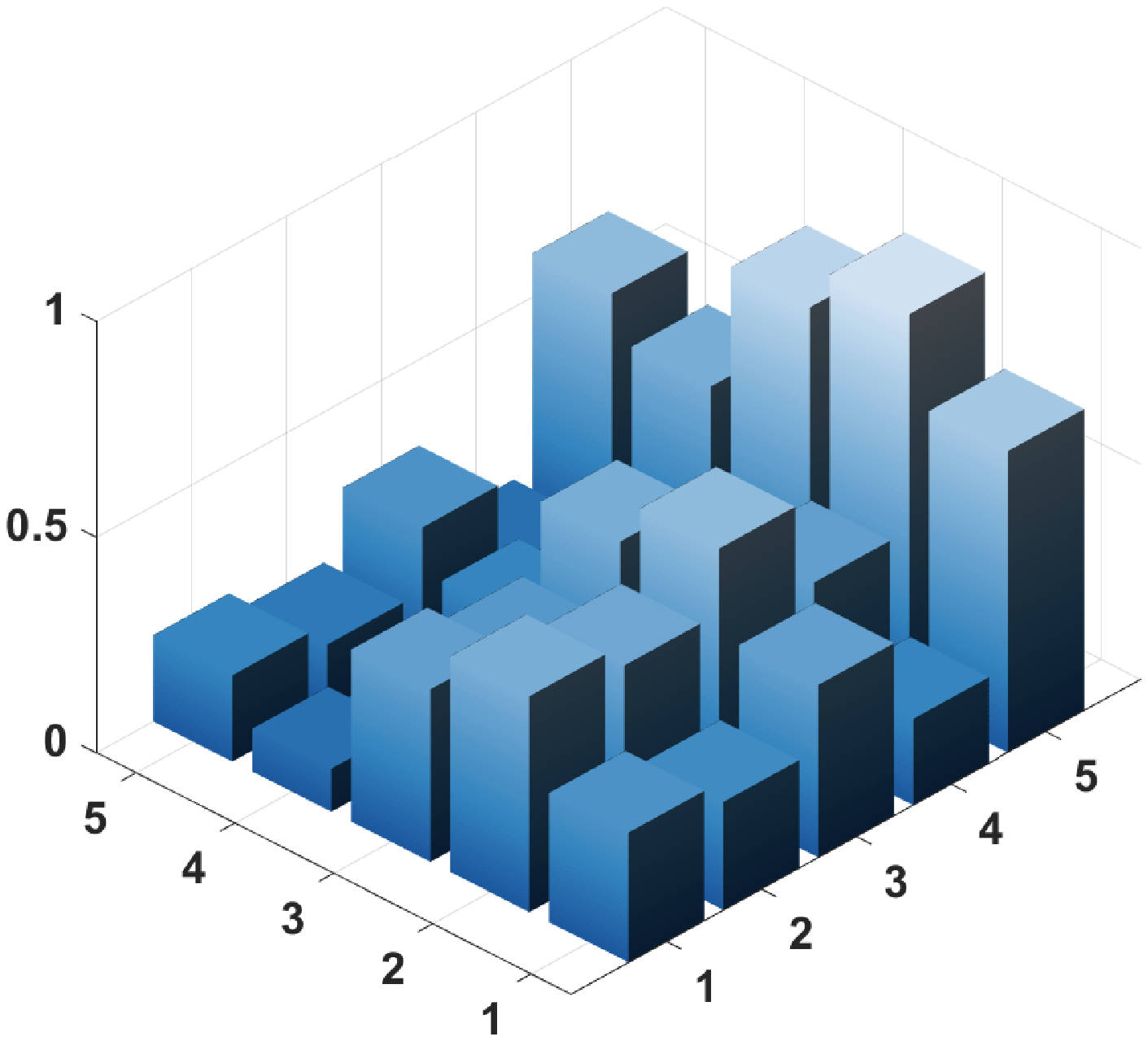}}
\subcaptionbox{\label{subfig:explanation_2d} The bird's-eye view of (\subref{subfig:explanation_3d})}
{\includegraphics[width=0.53\textwidth]{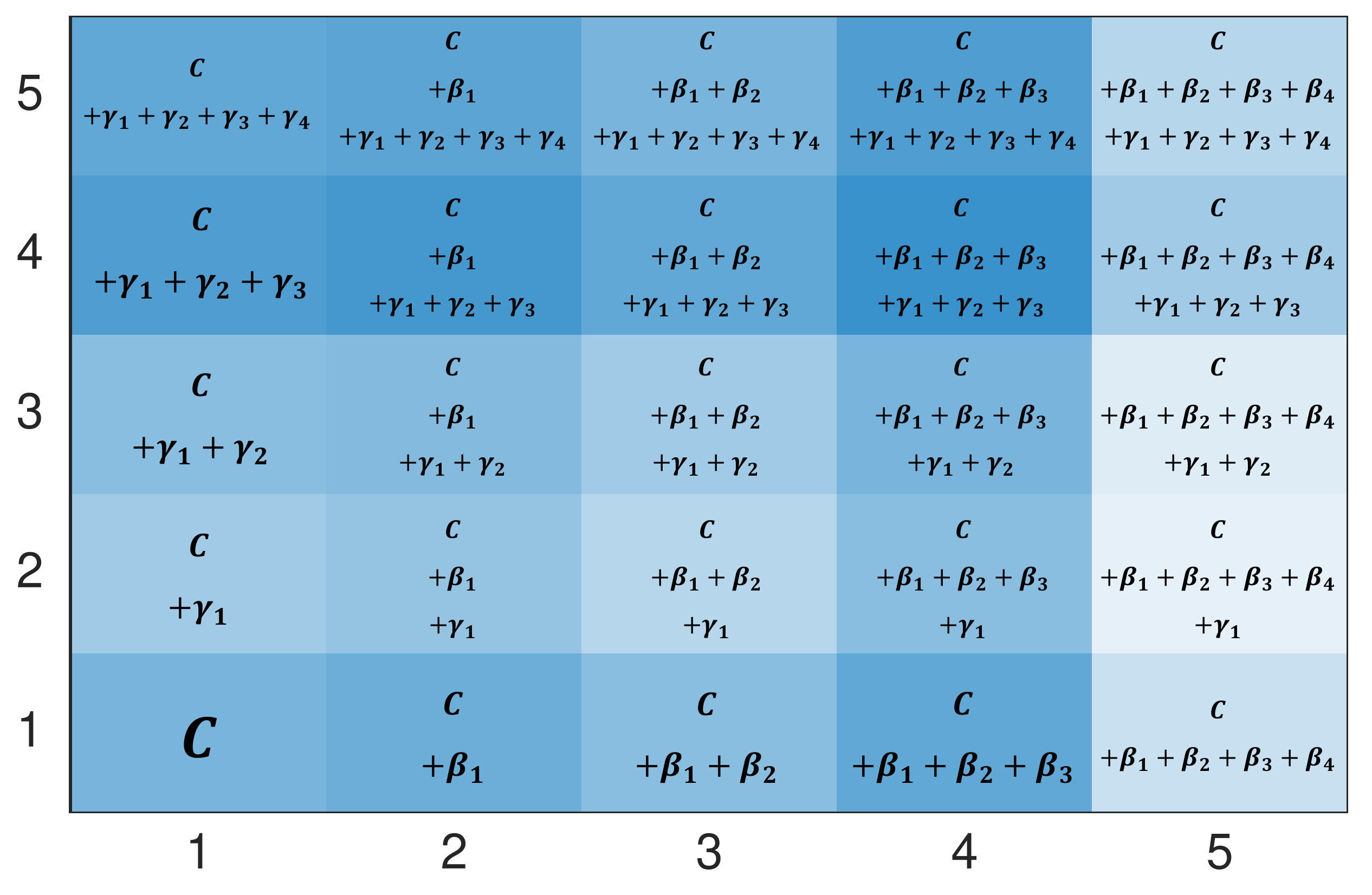}}
\caption{\label{fig:explanation_of_f} An example of the strong classifier $f$.}
\end{figure}

Fig. \ref{fig:explanation_of_f} is a toy example of the strong classifier $f$ with $M_1=5$ and $M_2=5$. Fig. \ref{fig:explanation_of_f}(\subref{subfig:explanation_3d}) is the graph of the function $f(x_1,x_2)$, and \ref{fig:explanation_of_f}(\subref{subfig:explanation_2d}) is the bird's-eye view of \ref{fig:explanation_of_f}(\subref{subfig:explanation_3d}). The darker the color is, the smaller the value of $f$ takes. The value of $f$ is written explicitly on \ref{fig:explanation_of_f}(\subref{subfig:explanation_2d}), which shows that, the values in row 2 are $\gamma_1$ greater than row 1, and the values in row 3 are $\gamma_2$ greater than row 2, and so on. Similarly, the values in column 2 are $\beta_1$ greater than column 1, and the values in column 3 are $\beta_2$ greater than column 2. All numbers in the grids increase or decrease \emph{the same} values, from left to right, and from bottom to top. It is the pattern of comonotonicity.

We have already shown in Fig. \ref{fig:not_comonotonic}(\subref{subfig:XOR_arrow}) that the XOR function is not comonotonic. Therefore, if the Bayes classifier of a population is the XOR function, it is impossible to give a good answer to the classification problem by training AdaBoost based on stumps. The conclusion in this section can be regarded as a generalization of Section \ref{subsection: XOR}.

In portfolio management, it is very common that factors may have interactions among each others. Hence, non-comonotonic populations are not rare. Although AdaBoost based on stumps can achieve good results in some areas, in financial studies, just based on stumps is far from reaching the desired goal. In Section \ref{sec:empirical}, we use empirical studies to show that, using deeper trees as base learners of AdaBoost is usually a better choice in portfolio management.

%% file: secs/empirical_results.tex
\section{Empirical studies}\label{sec:empirical}
In this section, we use the data of the Chinese A-share market to give empirical studies about the factor investing strategy based on AdaBoost. How to construct a stock factor strategy is an open problem with long history in portfolio management. From \citet{wang2012DB} invented the N-LASR to 
\citet{fievet2018trees} proposed a decision tree forecasting model, and 
\citet{gu2018_ML_asset_pricing} and \citet{dhondt2020} gave a comprehensive analysis of machine learning methods for the canonical problem of empirical asset pricing,
all of them agree that it may improve the strategy performance if the prediction model can dig out nonlinear and complex information. 

Our empirical studies have two goals. On the one hand, by selecting an optimal portfolio management strategy based on AdaBoost, we want to verify the general theoretical results about the interpolation and localization of AdaBoost in Section \ref{sec:ION and AdaBoost} and Section \ref{sec:XOR_counter}. On the other hand, we want to illustrate the good performance of the equal-weighted strategy based on AdaBoost.

In order to achieve the first goal, we give a sensitivity analysis about the depth of the base learners (decision trees) and the number of iterations of AdaBoost on the training set and the test set. We specifically explain the performance of AdaBoost that it can dig out useful information efficiently, as well as decrease the test error.

\subsection{Data}
The empirical data starts in June 2002 and ends in June 2017, 181 months in total. All stocks traded in the Chinese A-share market are included. 60 factors are used in our strategy. The data of the factor exposures and the monthly returns are downloaded from the Wind Financial Terminal\footnote{\url{https://www.wind.com.cn/en/Default.html}}. The 60 factors include not only the fundamental factors, but also the technical factors, such as the momentum and the turnover. All 60 factors are listed in Table \ref{tab:factors}. 

The original data has been preliminarily cleaned, but we still need to do some preprocessing before training. We remove all stocks which are not traded (or cannot be traded due to the limit-up or limit-down in the Chinese market) during the period we study. We remove the factors with over $10\%$ missing data, and fill in the missing data of other factors with 0. For each month, we assign the response variables $Y$ of all stocks according to their ranks of the next-month returns cross-sectionally. The response variables of the top $50\%$ stocks are $+1$, and that of the bottom $50\%$ stocks are $-1$.

We divide all data into a training set and a test set manually. Total 181 months' data is divided into two sets: the first 127 months' data (June 2002--December 2012) is taken as the training set, and the last 54 months' data (January 2013--June 2017) is taken as the test set. Then, the size of the training set is 193455 (sum of the stock numbers in all months), and the size of the test set is 133277. We use the training set to fit models, and then use the test set to evaluate the models and verify our conclusions in Section \ref{sec:ION and AdaBoost} and Section \ref{sec:XOR_counter}. 

\begin{table}[H]
\scriptsize
\centering
\caption{The 60 factors.}
\label{tab:factors}
\begin{tabular}{cccccc}
\hline
\hline
alr                & IR\_bps\_252 & IR\_netasset\_126    & IR\_net\_profit\_63    & IR\_roe\_252      & net\_assets          \\
amount\_21         & IR\_bps\_63  & IR\_netasset\_252    & IR\_oper\_rev\_126     & IR\_roe\_63       & oper\_rev\_ttm       \\
avg\_volume\_63    & IR\_eps\_126 & IR\_netasset\_63     & IR\_oper\_rev\_252     & IR\_totasset\_126 & pb                   \\
bps                & IR\_eps\_252 & IR\_net\_profit\_126 & IR\_oper\_rev\_63      & IR\_totasset\_252 & q\_eps               \\
IR\_bps\_126       & IR\_eps\_63  & IR\_net\_profit\_252 & IR\_roe\_126           & IR\_totasset\_63  & q\_grossprofitmargin \\
\hline
q\_netprofitmargin & q\_ps        & rt\_252              & tot\_assets            & ttm\_pcf          & turnover\_126        \\
q\_oper\_rev       & q\_roa       & rt\_63               & ttm\_eps               & ttm\_pe           & turnover\_21         \\
q\_orps            & q\_roe       & shr\_float2tot       & ttm\_grossprofitmargin & ttm\_ps           & turnover\_252        \\
q\_pcf             & rt\_126      & s\_dq\_mv            & ttm\_netprofitmargin   & ttm\_roa          & turnover\_63         \\
q\_pe              & rt\_21       & s\_val\_mv           & ttm\_orps              & ttm\_roe          & val\_float2tot \\ 
\hline
\hline
\end{tabular}
\end{table}

\subsection{The performance of the AdaBoost classifiers}
In this section, we analyze how the performance of the classifiers trained by AdaBoost varies with the two hyperparameters: the depth of the base learners (decision trees), and the number of iterations. Both hyperparameters are typically the source of overfitting in common sense. More specifically, we consider the following hyperparameters:
\begin{itemize}
    \item Max$\_$Depth: The maximum depth of the base learners (decision trees), takes $2, 4, 6, 8,$ and 10, respectively;
    \item N$\_$Steps: The number of iterations, takes $10, 20, 30, 40, 50, 100, 500,$ and 1000, respectively.
\end{itemize}

In order to analyze the influence of these two hyperparameters on the fitting ability of AdaBoost, we fix other parameters, and set the learning rates of all models as 0.1. Both in the training set and in the test set, we use the AUC (area under the ROC curve) to measure the performance of the models, which is a supplement to the usual error calculation.

The performance results of all models we studied are summarized in Table \ref{tab:ParametersAndResults}. Based on the results, we observe that:

\begin{itemize}
    \item The training/test AUC and the training/test error are consistent, since if the AUC is high, the error will be low in almost every scenarios.
    For example, when Max\_Depth $=2$ and N\_Steps $=10$ (the 1st model), the training AUC is 0.5412 while the training error is 0.4701; And when Max\_Depth $=2$ and N\_Steps $=20$ (the 2nd model), the training AUC is 0.5436 while the training error is 0.4700---They vary in different directions.

    \item 
    The training AUC increases monotonically as the complexity of the model increases. 
    Specifically, 
    from the first model to the last model, the complexity increases.
    Meanwhile, the training AUC increases from 0.5412 to 0.6828; The training error decreases too.

    \item The test AUC also almost increases monotonically as the complexity of the model increases. 
    For instance, when Max\_Depth $=2$ and N\_Steps $=10, \dots, 50$ (the 1st-5th models), the test AUC increases from 0.5433 to 0.5480; When N\_Steps $=20$ and Max\_Depth $=2, 4, 6$ (the 2nd, 7th and 12th models), the test AUC also increases from 0.5462 to 0.5490. 

    \item
    The changes of the test AUC are relatively small and stable with regard to that of the training AUC.
    For example, for the first 15 models, the test AUC changes from 0.5433 to 0.5513, while the training AUC changes from 0.5412 to 0.5946.
    It suggests that the test AUC around $ 54\%$ may be a stable threshold of the model, which reflects the ability of our methods to dig out the market information contained in our dataset.
    It is noteworthy that, $54\%$ is not a bad result in the Chinese stock market, according to the experience of the industry.
    
    \item 
    In the training set, the performance is more sensitive to Max\_Depth than to N\_Steps.
    In detail, given Max\_Depth $=2$, the training AUC changes from 0.5412 to 0.5533 for N\_Steps $=10,\dots,50$ (the 1st-5th models);
    However, given N\_Steps $=10$, the training AUC changes from 0.5412 to 0.5818 for Max\_Depth $=2, 4, 6$.
    
\end{itemize}

We find that, as the depth of the trees and the number of iterations increases, the AUC for the test set increases stably without significant change. We can conclude that, in these cases, the more iteration steps, the better the classifier, and the more complex the base learner trees, the better the classifier. 

\begin{table}[H]
\footnotesize
\centering
\caption{Model performance results.}
\label{tab:ParametersAndResults}
\begin{tabular}{c|cc|cc|cc}
\hline
\hline
\multirow{2}{*}{Model No.} & \multicolumn{2}{c|}{Hyperparameters} & \multicolumn{2}{c|}{Training Set} & \multicolumn{2}{c}{Test Set} \\
\cline{2-7} & Max\_Depth & N\_Steps & Training AUC & Training Error & Test AUC & Test Error  \\
\hline
1         & 2          & 10            & 0.5412 & 0.4701 & 0.5433 & 0.4713    \\
2         & 2          & 20            & 0.5436 & 0.4700 & 0.5462 & 0.4741    \\
3         & 2          & 30            & 0.5499 & 0.4665 & 0.5468 & 0.4728   \\
4         & 2          & 40            & 0.5511 & 0.4656 & 0.5476 & 0.4716  \\
5         & 2          & 50            & 0.5533 & 0.4636 & 0.5480 & 0.4714  \\
\hline
6         & 4          & 10            & 0.5628 & 0.4545 & 0.5463 & 0.4671   \\
7         & 4          & 20            & 0.5682 & 0.4515 & 0.5487 & 0.4697   \\
8         & 4          & 30            & 0.5699 & 0.4500 & 0.5489 & 0.4681    \\
9         & 4          & 40            & 0.5713 & 0.4505 & 0.5498 & 0.4669   \\
10        & 4          & 50            & 0.5723 & 0.4498 & 0.5500 & 0.4669  \\
\hline
11        & 6          & 10            & 0.5818 & 0.4418 & 0.5458 & 0.4715   \\
12        & 6          & 20            & 0.5870 & 0.4392 & 0.5490 & 0.4683 \\
13        & 6          & 30            & 0.5913 & 0.4353 & 0.5502 & 0.4675 \\
14        & 6          & 40            & 0.5930 & 0.4346 & 0.5506 & 0.4676  \\
15        & 6          & 50            & 0.5946 & 0.4338 & 0.5513 & 0.4670 \\
\hline
16        & 8          & 100           & 0.6300 & 0.4108 & 0.5519 & 0.4663  \\
17        & 8          & 500           & 0.6356 & 0.4071 & 0.5531 & 0.4659  \\
18        & 8          & 1000          & 0.6358 & 0.4071 & 0.5532 & 0.4659  \\
\hline
19        & 10         & 100           & 0.6731 & 0.3805 & 0.5486 & 0.4679  \\
20        & 10         & 500           & 0.6769 & 0.3780 & 0.5502 & 0.4677  \\
21        & 10         & 1000          & 0.6828 & 0.3740 & 0.5499 & 0.4681  \\
\hline
\hline
\end{tabular}
\end{table}

\subsection{Strategy results}

\begin{figure}[htpb]
\centering
\subcaptionbox{\label{subfig:longonly}Long-only strategy}
{\includegraphics[width=0.98\textwidth]{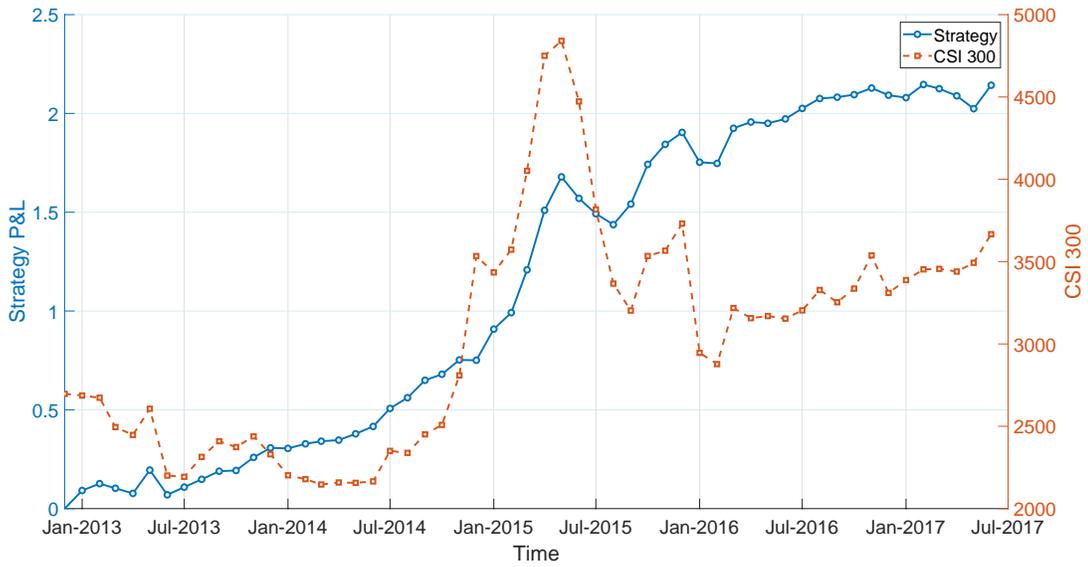}}
\subcaptionbox{\label{subfig:longshort}Long-short strategy}
{\includegraphics[width=0.98\textwidth]{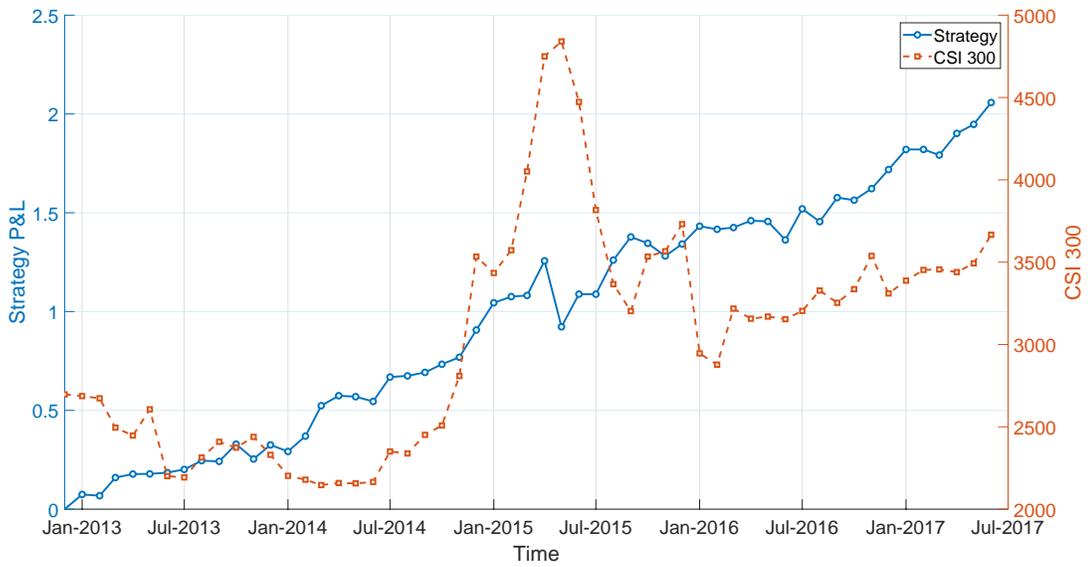}}
\caption{\label{fig:pnl} The P\&L of two strategies.}
\end{figure}

In this section, we use the classifier with good performance in the training set to establish equal-weighted strategies in the test set (January 2013--June 2017). We choose the 16th classifier in Table \ref{tab:ParametersAndResults}: Max$\_$Depth=8, and N$\_$steps=100. 

Our strategy is: At the beginning of each month, we treat the last month's exposures of the 60 factors for all stocks as the input of the classifier. Then, we use the output of the classifier to determine whether to take a long or a short position. For the long-short strategy, we long the top 50 stocks with the highest probability of going up, and short the bottom 50 stocks with the highest probability of going down; For the long-only strategy, we only long the top 50 stocks with the highest probability of going up. 

In Fig. \ref{fig:pnl}, we use the CSI 300 to denote the benchmark of the buy-and-hold strategy which we want to compare with. Fig. \ref{fig:pnl}(\subref{subfig:longshort}) and \ref{fig:pnl}(\subref{subfig:longonly}) show the P\&Ls of the long-only strategy and the long-short strategy, respectively. Both the transaction cost rates of taking a long position and a short position are set as 0.15\%. Table \ref{tab:StrategyPerformance} shows the performance summaries of our strategies, including the win rate, the Sharpe ratio, the average return, the standard deviation of return, and the maximum drawdown as different measures of the performance.

From the strategy performance, we have the following results:

\begin{itemize}
    \item The performance of our strategies is remarkable, although the AdaBoost classifier is fixed for the whole test period (January 2013--June 2017). The average annualized return for both the long-only and the long-short strategy are over 45\% during the test period. Particularly, the Chinese A-share market experienced a dramatic fluctuation in 2015, but our strategies (the blue solid line) still have a relatively robust return comparing with the CSI 300 benchmark (the orange dashed line). 
    
    \item The performance of our strategies is relatively stable, and even the maximum drawdown of the long-only strategy is much lower than the maximum drawdown of the benchmark CSI 300 (72.80\%). The maximum drawdown 24.18\% of the long-only strategy occurred in August 2015, and 33.41\% of the long-short strategy occurred in May 2015. 
    
\end{itemize}

Once again, we should note that the goal of these empirical studies is to find an optimal strategy by AdaBoost which is a complicated classifier in-sample, as well as to demonstrate that the performance of the complicated model is not bad out-of-sample. 

Overall, we can conclude from the empirical studies:
\begin{itemize}
    \item For the equal-weighted factor investing strategy, increasing the depth of the base learners (decision trees) and the number of iterations of AdaBoost do not significantly reduce the out-of-sample performance;
    \item For factor investing strategies, it is a feasible way for investors to use the complicated AdaBoost method to learn in the training set.
    
\end{itemize}

\begin{table}[H]
\centering
\caption{Strategy performance results.}
\label{tab:StrategyPerformance}
\begin{tabular}{c|ccccc}
\hline
Strategy & Win Rate & Sharpe Ratio & Average Return & Std. & Maximum Drawdown \\
\hline
long-only & 0.704 & 1.665 & 47.62\% & 28.60\% & 24.18\% \\
long-short & 0.704 & 1.568 & 45.74\%  & 29.17\% & 33.41\% \\
CSI 300 & 0.537 & 0.243 & 7.99\% & 32.83\% & 72.80\% \\
\hline
\end{tabular}
\end{table}

%% file: secs/conclusion.tex
\section{Conclusion}\label{sec:conclusion}

In order to implement machine learning in constructing equal-weighted portfolios interpretatively, our paper explains the success of AdaBoost and applies it in portfolio management.

We prove that a ``complex'' method, even interpolating, does not necessarily result in poor out-of-sample performance if it can lower the influence of the noise points. 
AdaBoost based on large trees without early stopping could be one of such methods, and we explain its success with our newly defined ION.

We illustrate the shortcomings of AdaBoost based on shallow trees from the perspective of its ability to dig out non-linear information, and emphasize the benefits of AdaBoost based on large trees.

Our empirical studies not only corroborate the theoretical results, but also give an effective approach to construct equal-weighted portfolios via machine learning.

Conclusively, we show that AdaBoost can minimize the influence of the noise points while interpolating the training set, and thus have a good out-of-sample performance. 
We confirm the conjectures in \citet{wyner2017AdaBoost_RandomForest} under a mathematical framework.
The empirical studies verify the potential applications of AdaBoost in portfolio management.

%% file: main.bbl
\begin{thebibliography}{}

\bibitem[\protect\citeauthoryear{Asness, Frazzini, Israel, Moskowitz, and
  Pedersen}{Asness et~al.}{2018}]{asness2018size}
Asness, C., A.~Frazzini, R.~Israel, T.~J. Moskowitz, and L.~H. Pedersen (2018).
\newblock Size matters, if you control your junk.
\newblock {\em Journal of Financial Economics\/}~{\em 129\/}(3), 479--509.

\bibitem[\protect\citeauthoryear{Breiman}{Breiman}{1998}]{breiman1998arcing}
Breiman, L. (1998).
\newblock Arcing classifier (with discussion and a rejoinder by the author).
\newblock {\em The Annals of Statistics\/}~{\em 26\/}(3), 801--849.

\bibitem[\protect\citeauthoryear{Breiman}{Breiman}{1999}]{breiman1999prediction}
Breiman, L. (1999).
\newblock Prediction games and arcing algorithms.
\newblock {\em Neural Computation\/}~{\em 11\/}(7), 1493--1517.

\bibitem[\protect\citeauthoryear{Breiman}{Breiman}{2001}]{breiman2001random}
Breiman, L. (2001).
\newblock Random forests.
\newblock {\em Machine Learning\/}~{\em 45\/}(1), 5--32.

\bibitem[\protect\citeauthoryear{Chen and Guestrin}{Chen and
  Guestrin}{2016}]{XGBoost}
Chen, T. and C.~Guestrin (2016).
\newblock {XGBoost}: A scalable tree boosting system.
\newblock In {\em Proceedings of the 22nd ACM SIGKDD International Conference
  on Knowledge Discovery and Data Mining}, KDD '16, New York, NY, USA, pp.\
  785--794. ACM.

\bibitem[\protect\citeauthoryear{{Cover} and {Hart}}{{Cover} and
  {Hart}}{1967}]{cover1967_1nn}
{Cover}, T. and P.~{Hart} (1967).
\newblock Nearest neighbor pattern classification.
\newblock {\em IEEE Transactions on Information Theory\/}~{\em 13\/}(1),
  21--27.

\bibitem[\protect\citeauthoryear{Creamer}{Creamer}{2012}]{creamer2012model}
Creamer, G. (2012).
\newblock Model calibration and automated trading agent for euro futures.
\newblock {\em Quantitative Finance\/}~{\em 12\/}(4), 531--545.

\bibitem[\protect\citeauthoryear{Creamer and Freund}{Creamer and
  Freund}{2010}]{creamer2010automated}
Creamer, G. and Y.~Freund (2010).
\newblock Automated trading with boosting and expert weighting.
\newblock {\em Quantitative Finance\/}~{\em 10\/}(4), 401--420.

\bibitem[\protect\citeauthoryear{Creamer and Freund}{Creamer and
  Freund}{2005}]{creamer2005AdaBoost}
Creamer, G.~G. and Y.~Freund (2005).
\newblock Using adaboost for equity investment scorecards.
\newblock {\em Howe School Research Paper\/}.
\newblock NIPS Workshop Machine Learning in Finance, 2005, Whistler, British
  Columbia, Canada.

\bibitem[\protect\citeauthoryear{DeMiguel, Garlappi, and Uppal}{DeMiguel
  et~al.}{2007}]{demiguel2007_1n}
DeMiguel, V., L.~Garlappi, and R.~Uppal (2007).
\newblock {Optimal versus naive diversification: How inefficient is the 1/N
  portfolio strategy?}
\newblock {\em The Review of Financial Studies\/}~{\em 22\/}(5), 1915--1953.

\bibitem[\protect\citeauthoryear{Devroye, Gy{\"o}rfi, and Lugosi}{Devroye
  et~al.}{1996}]{DGL1996}
Devroye, L., L.~Gy{\"o}rfi, and G.~Lugosi (1996).
\newblock {\em A Probabilistic Theory of Pattern Recognition}, Volume~31 of
  {\em Stochastic Modelling and Applied Probability}.
\newblock Springer-Verlag New York.

\bibitem[\protect\citeauthoryear{D'Hondt, {De Winne}, Ghysels, and
  Raymond}{D'Hondt et~al.}{2020}]{dhondt2020}
D'Hondt, C., R.~{De Winne}, E.~Ghysels, and S.~Raymond (2020).
\newblock Artificial intelligence alter egos: Who might benefit from
  robo-investing?
\newblock {\em Journal of Empirical Finance\/}~{\em 59}, 278--299.

\bibitem[\protect\citeauthoryear{Feng, Giglio, and Xiu}{Feng
  et~al.}{2020}]{feng2017zoo}
Feng, G., S.~Giglio, and D.~Xiu (2020).
\newblock Taming the factor zoo: a test of new factors.
\newblock {\em The Journal of Finance\/}~{\em 75\/}(3), 1327--1370.

\bibitem[\protect\citeauthoryear{Fi{\'e}vet and Sornette}{Fi{\'e}vet and
  Sornette}{2018}]{fievet2018trees}
Fi{\'e}vet, L. and D.~Sornette (2018).
\newblock Decision trees unearth return sign predictability in the s\&p 500.
\newblock {\em Quantitative Finance\/}~{\em 18\/}(11), 1797--1814.

\bibitem[\protect\citeauthoryear{Freund and Schapire}{Freund and
  Schapire}{1996}]{freund1996adaboost}
Freund, Y. and R.~E. Schapire (1996).
\newblock Experiments with a new boosting algorithm.
\newblock In {\em Proceedings of the Thirteenth International Conference on
  International Conference on Machine Learning}, ICML'96, San Francisco, CA,
  USA, pp.\  148--156. Morgan Kaufmann Publishers Inc.

\bibitem[\protect\citeauthoryear{Friedman, Hastie, and Tibshirani}{Friedman
  et~al.}{2000}]{friedman2000additive}
Friedman, J., T.~Hastie, and R.~Tibshirani (2000).
\newblock Additive logistic regression: a statistical view of boosting (with
  discussion and a rejoinder by the authors).
\newblock {\em The Annals of Statistics\/}~{\em 28\/}(2), 337--407.

\bibitem[\protect\citeauthoryear{Friedman}{Friedman}{2001}]{friedman2001greedy}
Friedman, J.~H. (2001).
\newblock Greedy function approximation: a gradient boosting machine.
\newblock {\em The Annals of Statistics\/}~{\em 29\/}(5), 1189--1232.

\bibitem[\protect\citeauthoryear{Gu, Kelly, and Xiu}{Gu
  et~al.}{2020}]{gu2018_ML_asset_pricing}
Gu, S., B.~Kelly, and D.~Xiu (2020).
\newblock {Empirical asset pricing via machine learning}.
\newblock {\em The Review of Financial Studies\/}~{\em 33\/}(5), 2223--2273.

\bibitem[\protect\citeauthoryear{Harvey}{Harvey}{2017}]{harvey2017presidential}
Harvey, C.~R. (2017).
\newblock Presidential address: the scientific outlook in financial economics.
\newblock {\em The Journal of Finance\/}~{\em 72\/}(4), 1399--1440.

\bibitem[\protect\citeauthoryear{Harvey, Liu, and Zhu}{Harvey
  et~al.}{2016}]{harvey2016and}
Harvey, C.~R., Y.~Liu, and H.~Zhu (2016).
\newblock … and the cross-section of expected returns.
\newblock {\em The Review of Financial Studies\/}~{\em 29\/}(1), 5--68.

\bibitem[\protect\citeauthoryear{Hastie, Tibshirani, and Friedman}{Hastie
  et~al.}{2009}]{book:esl}
Hastie, T., R.~Tibshirani, and J.~H. Friedman (2009).
\newblock {\em The Elements of Statistical Learning: Data Mining, Inference,
  and Prediction\/} (2 ed.).
\newblock Springer Series in Statistics. Springer, New York.

\bibitem[\protect\citeauthoryear{James}{James}{2003}]{james2003bias_variance}
James, G.~M. (2003).
\newblock Variance and bias for general loss functions.
\newblock {\em Machine Learning\/}~{\em 51\/}(2), 115--135.

\bibitem[\protect\citeauthoryear{Jobson and Korkie}{Jobson and
  Korkie}{1981}]{jobson1981}
Jobson, J.~D. and B.~M. Korkie (1981).
\newblock Performance hypothesis testing with the sharpe and treynor measures.
\newblock {\em The Journal of Finance\/}~{\em 36\/}(4), 889--908.

\bibitem[\protect\citeauthoryear{Ke, Meng, Finley, Wang, Chen, Ma, Ye, and
  Liu}{Ke et~al.}{2017}]{LGBM}
Ke, G., Q.~Meng, T.~Finley, T.~Wang, W.~Chen, W.~Ma, Q.~Ye, and T.-Y. Liu
  (2017).
\newblock Lightgbm: A highly efficient gradient boosting decision tree.
\newblock In I.~Guyon, U.~V. Luxburg, S.~Bengio, H.~Wallach, R.~Fergus,
  S.~Vishwanathan, and R.~Garnett (Eds.), {\em Advances in Neural Information
  Processing Systems}, Volume~30, pp.\  3146--3154. Curran Associates, Inc.

\bibitem[\protect\citeauthoryear{{L{\'o}pez de Prado}}{{L{\'o}pez de
  Prado}}{2018}]{de2018ml}
{L{\'o}pez de Prado}, M. (2018).
\newblock {\em Advances in Financial Machine Learning}.
\newblock Wiley.

\bibitem[\protect\citeauthoryear{Mease and Wyner}{Mease and
  Wyner}{2008}]{mease2008contrary}
Mease, D. and A.~Wyner (2008).
\newblock Evidence contrary to the statistical view of boosting.
\newblock {\em Journal of Machine Learning Research\/}~{\em 9\/}(Feb),
  131--156.

\bibitem[\protect\citeauthoryear{Michaud}{Michaud}{1989}]{michaud1989markowitz}
Michaud, R.~O. (1989).
\newblock The markowitz optimization enigma: Is `optimized' optimal?
\newblock {\em Financial Analysts Journal\/}~{\em 45\/}(1), 31--42.

\bibitem[\protect\citeauthoryear{Rasekhschaffe and Jones}{Rasekhschaffe and
  Jones}{2019}]{rasekhschaffe2019_ML_stock_selection}
Rasekhschaffe, K.~C. and R.~C. Jones (2019).
\newblock Machine learning for stock selection.
\newblock {\em Financial Analysts Journal\/}~{\em 75\/}(3), 70--88.

\bibitem[\protect\citeauthoryear{Schapire, Freund, Bartlett, and Lee}{Schapire
  et~al.}{1998}]{schapire1998boosting}
Schapire, R.~E., Y.~Freund, P.~Bartlett, and W.~S. Lee (1998).
\newblock Boosting the margin: a new explanation for the effectiveness of
  voting methods.
\newblock {\em The Annals of Statistics\/}~{\em 26\/}(5), 1651--1686.

\bibitem[\protect\citeauthoryear{Tu and Zhou}{Tu and Zhou}{2011}]{tu2011}
Tu, J. and G.~Zhou (2011).
\newblock Markowitz meets talmud: A combination of sophisticated and naive
  diversification strategies.
\newblock {\em Journal of Financial Economics\/}~{\em 99\/}(1), 204--215.

\bibitem[\protect\citeauthoryear{Wang, Luo, Cahan, Alvarez, Jussa, and
  Chen}{Wang et~al.}{2012}]{wang2012DB}
Wang, S., Y.~Luo, R.~Cahan, M.-A. Alvarez, J.~Jussa, and Z.~Chen (2012).
\newblock Signal processing: the rise of the machines.
\newblock {\em Deutsche Bank Quantitative Strategy, 5 June 2012\/}.

\bibitem[\protect\citeauthoryear{Wyner}{Wyner}{2003}]{wyner2003OnBA}
Wyner, A.~J. (2003).
\newblock On boosting and the exponential loss.
\newblock In {\em Proceedings of the Ninth International Workshop on Artificial
  Intelligence and Statistics, {AISTATS} 2003, Key West, Florida, USA, January
  3-6, 2003}.

\bibitem[\protect\citeauthoryear{Wyner, Olson, Bleich, and Mease}{Wyner
  et~al.}{2017}]{wyner2017AdaBoost_RandomForest}
Wyner, A.~J., M.~Olson, J.~Bleich, and D.~Mease (2017).
\newblock Explaining the success of adaboost and random forests as
  interpolating classifiers.
\newblock {\em Journal of Machine Learning Research\/}~{\em 18\/}(48), 1--33.

\bibitem[\protect\citeauthoryear{Zhang, Jacobsen, de~Roon, and Jiang}{Zhang
  et~al.}{2019}]{zhang2019boosting}
Zhang, H., B.~Jacobsen, F.~de~Roon, and F.~Jiang (2019).
\newblock Boosting portfolio choice in the big data era.
\newblock {\em SSRN 3495901\/}.

\end{thebibliography}
